\RequirePackage{tikz}
\documentclass[default,iicol]{sn-jnl}
\usepackage{amsmath,amsfonts,amsthm,amssymb}
\usepackage{array}
\usepackage[subrefformat=parens,labelformat=parens,caption=false,font=normalsize,labelfont=rm,textfont=rm]{subfig}
\usepackage{tikz}
\usepackage{hhline}
\usepackage{comment}
\usepackage{pgfplots}
\usepackage{booktabs}
\usepackage{arydshln}
\usepackage{color}
\usepackage{tabu}
\usepackage{cleveref}
\usepackage{url}
\usepackage{verbatim}
\usepackage{multirow}
\usepackage{graphicx}
\usepackage{natbib}
\setcitestyle{authoryear,round}

\jyear{2022}%

\def\fs{\scriptsize}
\crefname{section}{Section}{Sections}
\crefname{table}{Table}{Tables}
\crefname{figure}{Fig.}{Figs.}
\crefname{equation}{Eqn.}{Eqns.}
\crefname{algorithm}{Algorithm}{Algorithms}
\crefname{proposition}{Proposition}{Propositions}

\theoremstyle{thmstyleone}%
\newtheorem{theorem}{Theorem}
\newtheorem{proposition}[theorem]{Proposition}%

\theoremstyle{thmstyletwo}%

\theoremstyle{thmstylethree}%

\newcommand{\PreserveBackslash}[1]{\let\temp=\\#1\let\\=\temp}
\newcolumntype{C}[1]{>{\PreserveBackslash\centering}p{#1}}
\newcolumntype{R}[1]{>{\PreserveBackslash\raggedleft}p{#1}}
\newcolumntype{L}[1]{>{\PreserveBackslash\raggedright}p{#1}}

\begin{document}

\title[Descriptor Distillation]{Descriptor Distillation: a
  Teacher-Student-Regularized Framework for
  Learning Local Descriptors}


\author[1,2]{\fnm{Yuzhen} \sur{Liu}}\email{liuyuzhen22@mails.ucas.ac.cn}

\author*[1,2]{\fnm{Qiulei} \sur{Dong}}\email{qldong@nlpr.ia.ac.cn}

\affil[1]{\orgdiv{State Key Laboratory of Multimodal Artificial Intelligence Systems}, \orgname{CASIA}, \orgaddress{\city{Beijing}, \postcode{100190}, \country{China}}}

\affil[2]{\orgdiv{School of Artificial Intelligence}, \orgname{UCAS}, \orgaddress{\city{Beijing}, \postcode{100049}, \country{China}}}


\abstract{Learning a fast and discriminative patch descriptor is a challenging topic in
  computer vision. Recently, many existing works focus on training various descriptor
  learning networks by minimizing a triplet loss (or its variants), which is expected
  to decrease the distance between each positive pair and increase the distance between
  each negative pair. However, such an expectation has to be lowered due to the
  non-perfect convergence of network optimizer to a local solution. Addressing this
  problem and the open computational speed problem, we propose a
  \textit{Des}criptor \textit{Dis}tillation framework for local descriptor learning,
  called DesDis, where a student model gains knowledge from a pre-trained teacher
  model, and it is further enhanced via a designed teacher-student regularizer.
  This teacher-student regularizer is to constrain the difference between the positive
  (also negative) pair similarity from the teacher model and that from the student
  model, and we theoretically prove that a more effective student model could be
  trained by minimizing a weighted combination of the triplet loss and this regularizer,
  than its teacher which is trained by minimizing the
  triplet loss singly. Under the proposed DesDis, many existing descriptor networks
  could be embedded as the teacher model, and accordingly, both equal-weight and
  light-weight student models could be derived, which outperform their teacher in
  either accuracy or speed. Experimental results on 3 public datasets demonstrate
  that the equal-weight student models, derived from the proposed DesDis framework
  by utilizing three typical descriptor learning networks as teacher models, could
  achieve significantly better performances than their teachers and several other
  comparative methods. In addition, the derived light-weight models could achieve
  8 times or even faster speeds than the comparative methods under similar patch
  verification performances.}




\maketitle

\section{Introduction}\label{sec:introduction}
Local descriptor learning plays an important role in various visual tasks, such as
image retrieval \citep{imret:bow,SOLAR}, panorama stitching
\citep{pano,pano:Brown07automaticpanoramic}, structure-from-motion
\citep{sfm:colmap,sfm:dong} and multi-view stereo
\citep{mvs:colmap,mvs:7780961}. The existing descriptors in literature could be roughly
divided into two categories: hand-crafted descriptors and learning-based descriptors.
Early works
\citep{SIFT,DSP-SIFT,SURF} mainly focused on hand-crafted ones according to researchers'
expertise. Recently, learning-based descriptors
\citep{TNet,HardNet,PCA-SIFT,SOSNet,HyNet,GeoDesc,invardesc,sfm:CrossDesc,R2D2,D2Net,zhao2022alike,invfeat,sun2021loftr},
particularly DNN(Deep Neural Network)-based descriptors, have shown a significant priority
to their hand-crafted counterparts. A comprehensive overview could be found in \citep{2021survey}.

Different from the hand-crafted descriptors, DNN-based descriptors are automatically
learned via deep neural networks, which are trained  by utilizing large-scale local patch
datasets \citep{HPatches,BrownDataset} with ground truth correspondences.
In recent years, a lot of descriptor learning networks have
employed a triplet loss \citep{TNet,HardNet} or its variants \citep{SOSNet,HyNet}, that
are expected to enforce the distances of positive pairs of patch descriptors to be
smaller while the distances of negative pairs to be larger. However, due to the
complexity of these descriptor learning networks, their used optimizers generally
converge to local minima at the training stage \citep{wang2017train,yu2018correcting,faghri2018vse++}
(this is to say, the distances of the positive pairs of descriptors learned by
these networks are not sufficiently smaller, while the distances of the
negative pairs of descriptors learned by these networks are
not sufficiently larger), resulting in lower-than-expected performances.

In order to alleviate the above local convergence problem and additionally speed
up the descriptor inference, we propose a \emph{Des}criptor \emph{Dis}tillation
framework for local descriptor learning, called DesDis, which is inspired by the
model compression ability of the knowledge distillation technique in some other
visual tasks, such as image classification \citep{HintonDistillation,classification},
object detection \citep{obj1,obj2} and face recognition \citep{face1,face2022}.
It has to be pointed out that knowledge distillation is originally a model
compression technique, which aims to transfer knowledge from a pre-trained teacher
model to a smaller student model, but it is not deliberately designed for improving
model accuracy. Hence, under the proposed framework, given a teacher model
(which could be an arbitrary existing descriptor learning network), a teacher-student
regularizer is firstly designed to alleviate the aforementioned local
convergence problem and enhance the ability of the student model by
minimizing the difference between the positive (also negative)
pair similarity under the teacher model and that under the student model. Then,
we prove theoretically in \cref{sec:analysis} that by minimizing a weighted
combination of the
triplet loss and the designed regularizer, the distances of positive
(or negative) descriptor pairs from the student
model could be smaller (or larger) than those from its teacher model.
Consequently, different student models with different compression rates for
learning either more discriminative or faster local descriptors could be
naturally derived under the proposed DesDis framework by utilizing different
existing descriptor learning networks as teacher models.

In sum, our main contributions include:
\begin{itemize}
  \item[$\bullet$]
    We propose the DesDis framework for local descriptor learning through knowledge
    distillation, where many existing descriptor learning networks could be seamlessly
    embedded as the teacher models.
  \item[$\bullet$]
    We explore the teacher-student regularizer under the proposed DesDis framework,
    which is helpful to further decrease (or increase) the distances of positive
    (or negative) pairs of descriptors output by the student model in comparison
    to the teacher model at the training stage. In addition, we give a theoretical
    proof of the effectiveness of this regularizer.
  \item[$\bullet$]
    Given an arbitrary teacher model under the proposed DesDis framework, not only
    a more discriminative equal-weight student model, but also a set of light-weight
    student models that achieve a trade-off between computational accuracy and speed,
    could be derived, whose effectiveness has been demonstrated by the experimental results
    in \cref{sec:experiments}.
\end{itemize}

The rest of the paper is organized as follows. \cref{sec:related} gives a review
on DNN-based local descriptors and knowledge distillation in other visual tasks.
\cref{sec:method} introduces the framework in detail. Experimental results are
reported in \cref{sec:experiments}. \cref{sec:conclusion} concludes the paper.

\section{Related Work}\label{sec:related}
Here, we firstly review some DNN-based methods for descriptor learning in literature.
Then, considering that the proposed framework employs the knowledge distillation
technique, we also give a review on knowledge-distillation-based methods for handling
other visual tasks.

\subsection{DNN-based Local Descriptors}
As discussed in \cref{sec:introduction}, in recent years, DNN-based methods
for descriptor learning have shown a significant priority to the early hand-crafted methods
\citep{SIFT,SURF,PCA-SIFT,DSP-SIFT} in literature.
MatchNet \citep{MatchNet} employed a siamese architecture, where one branch was used
for mapping a patch to a feature representation and the other was used for measuring the
similarity of features.
\citet{TNet} used a triplet margin loss with anchor swap to construct triplets.
\citet{L2Net} proposed the L2Net for learning local descriptors,
which utilized a fully convolutional architecture.
\citet{HardNet} adopted the `hardest in the batch' strategy to sample negative pairs.
\citet{DOAP} directly optimized a ranking-based retrieval performance
metric for learning local descriptors.
\citet{GeoDesc} integrated the geometric constraints from multi-view
reconstructions, which could benefit the learning process in terms of data
generalization, data sampling and loss computation.
\citet{SOSNet} jointly used the traditional first order similarity loss term
and a designed second order similarity regularizer to train a descriptor network.
\citet{HyNet} showed theoretically and empirically that a
hybrid similarity measure could balance the gradients from negative and positive
samples. In addition, unlike \citep{TNet,DOAP,HardNet,SOSNet,HyNet} where a large
number of local image patches are used as inputs for network training, some works
\citep{D2Net,R2D2,aslfeat,disk} also investigated to use whole images as input for
training and then output the descriptors of all the key points together in each
input image.

It is worth noting that the aforementioned works
\citep{TNet,L2Net,DOAP,HardNet,GeoDesc,SOSNet,HyNet} aimed to train descriptor
learning networks with a set of local patches by simultaneously minimizing
the distances of positive descriptor pairs and maximizing the distances of negative
descriptor pairs with a triplet loss or its variants, but these works were
prone to obtaining local solutions in practice due to the networks' complexity.
This issue motivates us to design a more effective regularizer
(i.e. the teacher-student regularizer which would be described in detail in
\cref{sec:regularizer}), which could ensure that the distances of positive
descriptor pairs would be further decreased while those of negative descriptor pairs
would be further increased in both theory and practice.

\begin{figure*}[t!]
  \centering
  \includegraphics[width=0.7\textwidth]{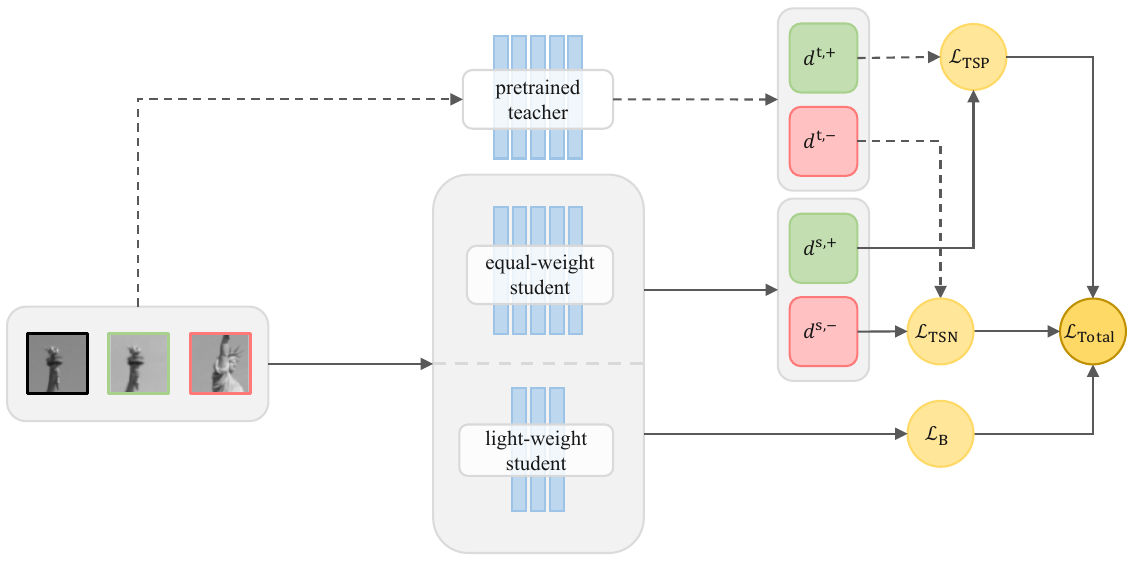}
  \caption{The pipeline of the proposed DesDis framework. $d^{\mathrm{t},+}$
    and $d^{\mathrm{t},-}$ are the distances of positive and negative descriptor
    pairs from the teacher model. $d^{\mathrm{s},+}$ and $d^{\mathrm{s},-}$ are
    the distances of positive and negative descriptor pairs from the student model.
    $\mathcal{L}_\mathrm{B}$ is a triplet loss variant. $\mathcal{L}_\mathrm{TSP}$ and
    $\mathcal{L}_\mathrm{TSN}$ are the two forms of the teacher-student regularizer.}
  \label{fig:pipeline}
  \vspace{-0.1in}
\end{figure*}

\subsection{Knowledge Distillation in Other Visual Tasks}
Knowledge distillation is a model compression technique, where a student network is
trained to mimic a pre-trained teacher model. The student model could benefit from
extra supervisory signal by the soft information from the teacher network. It was
firstly proposed for image classification by
\citet{HintonDistillation}. Following this seminal work, the
knowledge distillation technique has been extended to handle different visual tasks,
such as object detection \citep{obj1,obj2}, face recognition \citep{face1,face2022},
image segmentation \citep{SegmentationDistillation,SegmentationInter},
pose estimation \citep{pose1,pose2}, \textit{etc}.
For example, \citet{obj1} explored a knowledge distillation
method to transfer the knowledge of unseen categories for object detection.
\citet{face2022} proposed an evaluation-oriented
knowledge distillation method for deep face recognition, which could reduce
the performance gap between the teacher and student models during training.
\citet{pose2} proposed a fast pose distillation model training
method enabling to more effectively train small human pose CNN networks.

Here, the following two points have to be explained:
\begin{itemize}
  \item[(i)]
    It is noted that \citet{featdist:9423453} introduced the concept of
    `feature distillation' for descriptor learning, which is to directly learn a
    descriptor from the features extracted from the intermediate layers of a pretrained
    convolutional network, and it is significantly different from the concept of
    knowledge distillation used in our work and the aforementioned
    works \citep{HintonDistillation,obj1,FaceDistillation:Neuron,pose1},
    which is to pursue a student model from a given teacher model.
  \item[(ii)] As discussed above, knowledge distillation is originally a model
    compression technique, but it is not deliberately designed for improving model
    accuracy. Hence, although the original knowledge distillation technique
    \citep{HintonDistillation} is seemingly able to be used for handling the descriptor
    learning task, it is intrinsically unable to guarantee a student model with a
    higher accuracy, even with a comparable accuracy to its teacher model. This is
    the issue that also motivates us to explore the teacher-student regularizer in the
    following section.
\end{itemize}

\section{Methodology}\label{sec:method}

In this section, we  propose the DesDis framework with the designed teacher-student
regularizer for local descriptor learning. Firstly, we describe the pipeline of the
DesDis framework. Then, we present the teacher-student regularizer as well as the
total loss function. Finally, we give a theoretical analysis on the designed
teacher-student regularizer.

\subsection{The DesDis Framework}\label{sec:framework}
The proposed DesDis framework utilizes a knowledge distillation strategy, whose
pipeline is shown in \cref{fig:pipeline}. As seen from this figure,
the DesDis framework consists of a pre-trained teacher model, a student model, and
a loss function for training the student model. In principle, many existing
networks \citep{TNet,L2Net,HardNet,DOAP,HyNet} for descriptor learning
could be straightforwardly used as teacher models, and the goal
of DesDis is to pursue such a student model that could not only gain
knowledge from the pre-trained teacher model, but also achieve better
performances by introducing a teacher-student regularizer.

In this work, we use or design the following two kinds of student models:
\begin{itemize}
  \item [(i)] \textbf{Equal-weight student model:}
        For a pre-trained teacher model, we straightforwardly adopt its architecture as the
        corresponding equal-weight student model.

  \item [(ii)] \textbf{Light-weight student model:}
        Here, a light-weight student model means a smaller model with fewer parameters than a given
        teacher model. Considering that the number of the convolutional layers in many
        state-of-the-art descriptor networks \citep{L2Net,HardNet,SOSNet,HyNet} are no less than 7,
        we design a set of 5-convolutional-layer networks with different numbers of
        channels as light-weight student models, whose architecture is shown in
        \cref{fig:KDNet}. These models are denoted as DesDis-$D$, where $D$
        represents the channel number in the first convolutional layer, and is set to
        \{8,16,24,32\} respectively in our work.
\end{itemize}

\begin{figure}
  \begin{center}
    \includegraphics[width=0.17\textwidth]{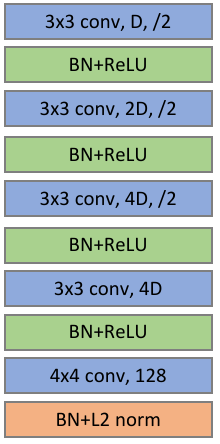}
  \end{center}
  \caption{Architectures of the design light-weight model DesDis-$D$
    which consists of 5 convolutional layers.
    `$\backslash$2' denotes strided convolution with a stride of 2.}
  \label{fig:KDNet}
  \vspace{-0.1in}
\end{figure}

In the following parts, we will describe the designed teacher-student regularizer as
well as the total loss function $\mathcal{L}_\textrm{Total}$ in detail, and provide the
theoretical analysis on the designed teacher-student regularizer.

\subsection{Teacher-Student Regularizer}\label{sec:regularizer}
As discussed in \cref{sec:introduction}, many existing descriptor learning
methods \citep{TNet,L2Net,HardNet,SOSNet,HyNet} are trained by minimizing the
triplet loss or its variants, and the distances of the positive (or negative) pairs
of their learned descriptors are not sufficiently smaller (or larger). This issue
motivates us to design the following teacher-student (TS) regularizer under the
proposed descriptor distillation framework consisting of a pretrained teacher model
, which could be one of the aforementioned works \citep{TNet,L2Net,HardNet,SOSNet,HyNet},
and a student model, so that the distances of positive descriptor pairs are further
decreased and the distances of negative descriptor pairs are increased in the
student model, comparing with the pre-trained teacher model.
The proposed TS regularizer neither maximizes the similarity of
each positive pair nor minimizes the similarity of each negative pair, but it
aims to minimize the difference between the positive (also negative) pair
similarity from the teacher model and that from the student model.

Specifically, given $N$ anchor sample patches, $N$ matching patches and $N$
non-matching patches, their corresponding descriptors learned by the teacher model
are denoted as $\{x_i^\mathrm{t}\}_{i=1}^N$, $\{x_i^{\mathrm{t}, +}\}_{i=1}^N$
and $\{x_i^{\mathrm{t}, -}\}_{i=1}^N$
respectively, and similarly, their corresponding descriptors learned by the student model are
denoted as $\{x_i^\mathrm{s}\}_{i=1}^N$, $\{x_i^{\mathrm{s}, +}\}_{i=1}^N$ and $\{x_i^{\mathrm{s}, -}\}_{i=1}^N$
respectively. The TS regularizer has dual forms for handling positive
and negative pairs respectively. The form for positive pairs is formulated as:
\begin{equation}\label{eq:LTSP}
  \mathcal{L}_\mathrm{TSP}=\frac{1}{N}\sum_{i=1}^N(d^{\mathrm{t},+}_i-d^{\mathrm{s},+}_i)^2\\
\end{equation}
\noindent where $d^{\mathrm{t},+}_i=\|x^\mathrm{t}_i-x^{\mathrm{t},+}_i\|_2$,
$d^{\mathrm{s},+}_i=\|x^\mathrm{s}_i-x^{\mathrm{s},+}_i\|_2$ are the Euclidean distances of the $i$-th
positive pair of descriptors learned by the teacher
and student models respectively (`$\|\cdot\|_2$' denotes the $L_2$ norm).

Similarly, the form of the TS regularizer for negative pairs is formulated as:
\begin{equation}\label{eq:LTSN}
  \mathcal{L}_\textrm{TSN}=\frac{1}{N}\sum_{i=1}^N(d^{\mathrm{t},-}_i-d^{\mathrm{s},-}_i)^2\\
\end{equation}
\noindent where $d^{\mathrm{t},-}_i=\|x^\mathrm{t}_i-x^{\mathrm{t},-}_i\|_2$,
$d^{\mathrm{s},-}_i=\|x^\mathrm{s}_i-x^{\mathrm{s},-}_i\|_2$ are the Euclidean distances of the $i$-th negative pair
of descriptors learned by the teacher and student models respectively.

\subsection{Total Loss Function}

It is noted from \cref{eq:LTSP} and \cref{eq:LTSN} that the designed TS
regularizer requires neither each positive pair to be closer, nor each negative pair
to be more distant. Hence, this regularizer is not used singly for training the
student model under the proposed DesDis framework. Instead, it is used jointly
with the triplet-like loss term $\mathcal{L}_\textrm{B}$ of the pre-trained teacher model.

In this work, the total loss function $\mathcal{L}_\textrm{Total}$ for training the student model
is a weighted combination of the loss term $\mathcal{L}_\textrm{B}$ of the teacher model and the two
aforementioned dual forms of the TS regularizer as:
\begin{equation}\label{eq:Lobj}
  \mathcal{L}_\textrm{Total}=\mathcal{L}_\textrm{B}+\alpha_\textrm{p} \mathcal{L}_\textrm{TSP}+\alpha_\textrm{n}\mathcal{L}_\textrm{TSN}
\end{equation}
\noindent where $\alpha_\textrm{p}$ and $\alpha_\textrm{n}$ are two preset weight parameters.
$\mathcal{L}_\textrm{B}$ could be the classic triplet loss or its variants in the
existing works \citep{HardNet,SOSNet,HyNet}.


\subsection{Theoretical Analysis}\label{sec:analysis}
In this subsection, we give a theoretical proof that once the regularizer is used, the
distances of positive (or negative) pairs of descriptors in  the student model are
smaller (or larger) than those in the teacher model at the training stage.

Here, we analyze the fundamental case where the classic triplet loss
$\mathcal{L}_\mathrm{T}$ in \cref{eq:Lt} is used as the loss term
$\mathcal{L}_\textrm{B}$ in the total loss
function (\ref{eq:Lobj}).

The  classic triplet loss $\mathcal{L}_\mathrm{T}$ is formulated as:
\begin{equation}\label{eq:Lt}
  \mathcal{L}_\mathrm{T}=\frac{1}{N}\sum_{i=1}^N{\rm max}(0,m+d_i^+-d_i^-)
\end{equation}
\noindent where $m$ is a hyperparameter.

For a pre-trained teacher model which has been trained by minimizing the triplet
loss $\mathcal{L}_\mathrm{T}$, the distances $\{d_i^{\mathrm{t},+}\}_{i=1}^N$ of
$N$ positive descriptor pairs and the distances $\{d_i^{\mathrm{t},-}\}_{i=1}^N$
of $N$ negative descriptor pairs in the
teacher model could be straightforwardly obtained. Then, we give the following
proposition on the student model which is trained by minimizing the total loss
function (\ref{eq:Lobj}) with $\mathcal{L}_\mathrm{T}$:

\begin{figure*}[h]
  \captionsetup[subfigure]{justification=centering}
  \def\sca{0.234}
  \centering
  \subfloat[Results from DesDis-HardNet on Brown]{
    \includegraphics[width=\sca\textwidth]{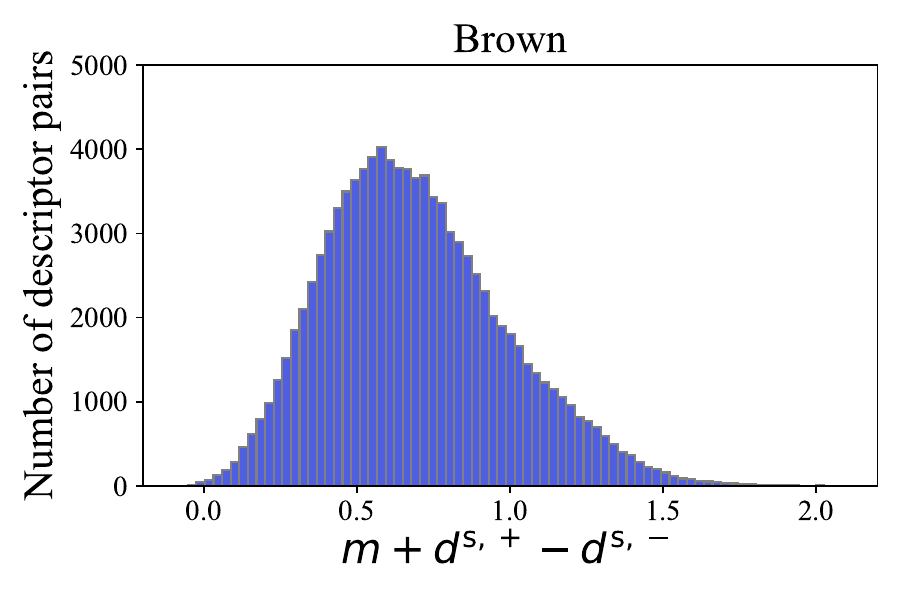}
    \label{fig:desdis_hardnet_brown}
  }
  \subfloat[Results from DesDis-HardNet on HPatches]{
    \includegraphics[width=\sca\textwidth]{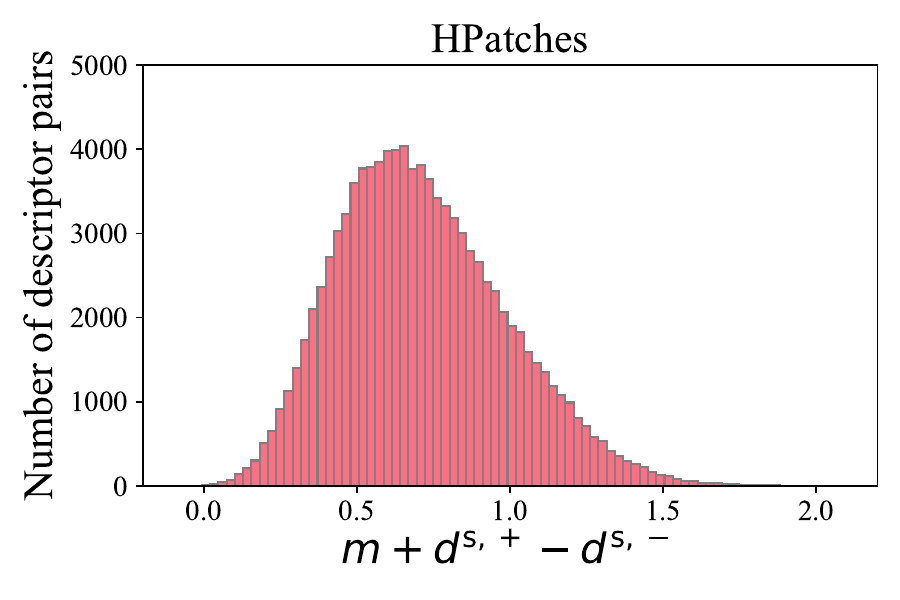}
    \label{fig:desdis_hardnet_hpatches}
  }
  \subfloat[Results from DesDis-32 on Brown]{
    \includegraphics[width=\sca\textwidth]{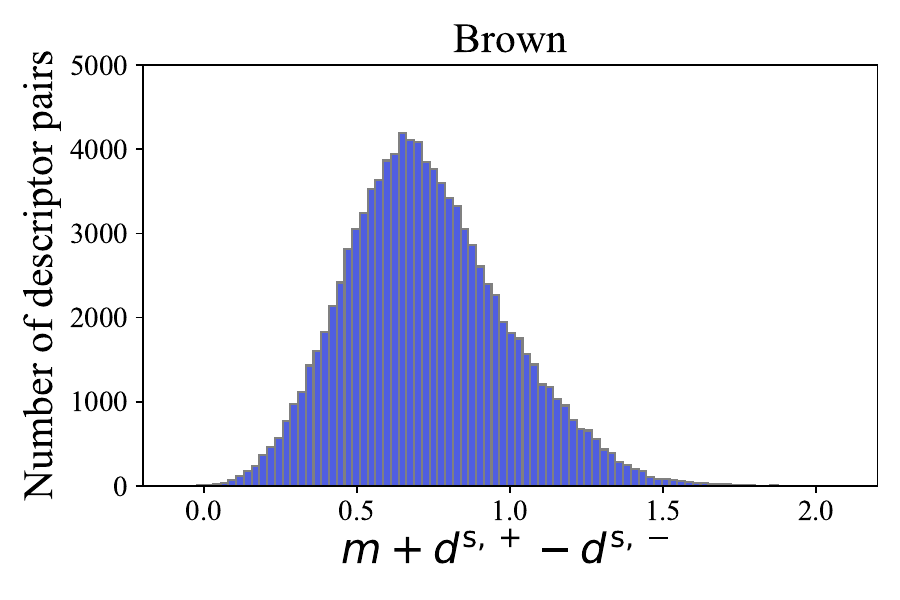}
    \label{fig:desdis_sift_brown}
  }
  \subfloat[Results from DesDis-32 on HPatches]{
    \includegraphics[width=\sca\textwidth]{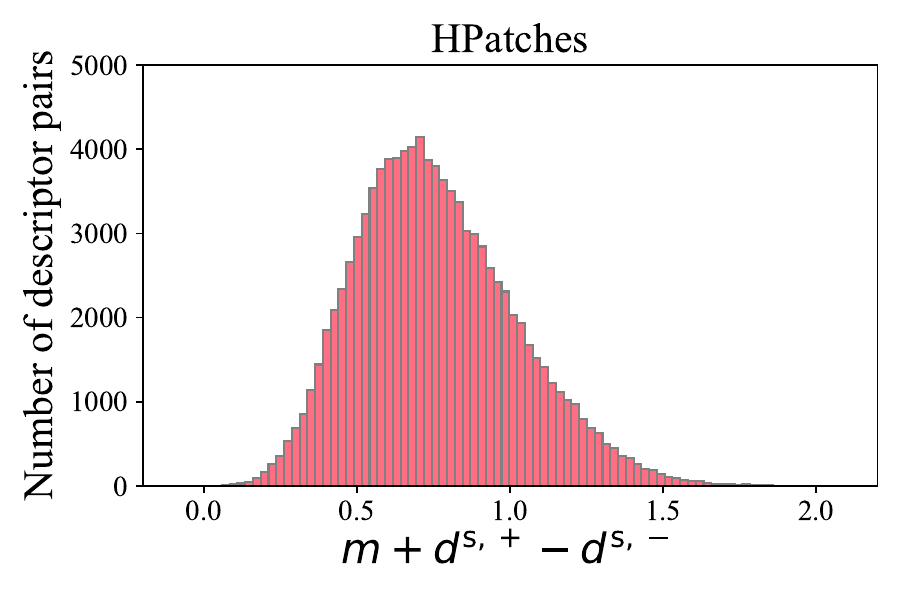}
    \label{fig:desdis_sift_hpatches}
  }
  \caption{
  The distribution of
  (a) `$m+d^{\mathrm{s},+}-d^{\mathrm{s},-}$' from DesDis-HardNet on Brown,
  (b) `$m+d^{\mathrm{s},+}-d^{\mathrm{s},-}$' from DesDis-HardNet on HPatches,
  (c) `$m+d^{\mathrm{s},+}-d^{\mathrm{s},-}$' from DesDis-32 (using SIFT as teacher)
  on Brown, and
  (d) `$m+d^{\mathrm{s},+}-d^{\mathrm{s},-}$' from DesDis-32 (using SIFT as teacher) on HPatches.
  `$d^{\mathrm{s},+}$' and `$d^{\mathrm{s},-}$' denote the distance of positive
  and negative pairs respectively from the student.
  }
\end{figure*}

\begin{proposition}\label{prop}
  Given  the distances $\{d_i^{\mathrm{t},+}\}_{i=1}^N$ of N positive descriptor pairs and
  the distances $\{d_i^{\mathrm{t},-}\}_{i=1}^N$ of N negative descriptor pairs in the
  teacher model, under the condition that $m+d_i^{\mathrm{s},+}-d_i^{\mathrm{s},-}>0,i=1,2,\cdots,N$,
  the optimal solution to the total loss function (\ref{eq:Lobj}) with $\mathcal{L}_\mathrm{T}$ satisfies:
  $$d_i^{\mathrm{s},+} < d_i^{\mathrm{t},+}, \quad d_i^{\mathrm{s},-} > d_i^{\mathrm{t},-}, i = 1,2,...,N$$
\end{proposition}

\begin{proof}
  By replacing $\mathcal{L}_\textrm{B}$ in \cref{eq:Lobj} with the triplet loss $\mathcal{L}_\mathrm{T}$ in
  \ref{eq:Lt}, under the condition that $m+d_i^{\mathrm{s},+}-d_i^{\mathrm{s},-}>0,i=1,2,\cdots,N$,
  the total loss function could be re-formulated as:
  \begin{equation}
    \begin{aligned}
      \mathcal{L}_\textrm{Total}
       & = \frac{1}{N}\sum_{i=1}^N[{\rm max}(0,m+d_i^{\mathrm{s},+}-d_i^{\mathrm{s},-}) +                                          \\
       & \alpha_\textrm{p}(d^{\mathrm{t},+}_i-d^{\mathrm{s},+}_i)^2 + \alpha_\textrm{n} (d^{\mathrm{t},-}_i-d^{\mathrm{s},-}_i)^2] \\
       & = \frac{1}{N}\sum_{i=1}^N[(m+d_i^{\mathrm{s},+}-d_i^{\mathrm{s},-}) +                                                     \\
       & \alpha_\textrm{p}(d^{\mathrm{t},+}_i-d^{\mathrm{s},+}_i)^2 + \alpha_\textrm{n} (d^{\mathrm{t},-}_i-d^{\mathrm{s},-}_i)^2]
    \end{aligned}\label{eq:Ltotal}
  \end{equation}

  $\mathcal{L}_\textrm{Total}$ is a convex function of $d_i^{\mathrm{s},+}$ and $d_i^{\mathrm{s},-},i=1,2,\cdots,N$,
  hence its optimal solution is achieved if and only if its first-order derivatives
  with respect to both $d_i^{\mathrm{s},+}$ and $d_i^{\mathrm{s},-}$  are equal to zero.

  The first-order partial derivatives of $\mathcal{L}_\textrm{Total}$ with respect
  to $d_i^{\mathrm{s},+}$ and $d_i^{\mathrm{s},-}$ are:
  \begin{equation}\label{eq:derivatives}
    \begin{aligned}
      \frac{\partial\mathcal{L}_\textrm{Total}}{\partial d_i^{\mathrm{s},+}} & =1+2\alpha_\textrm{p}d_i^{\mathrm{s},+}-2\alpha_\textrm{p}d_i^{\mathrm{t},+}  \\
      \frac{\partial\mathcal{L}_\textrm{Total}}{\partial d_i^{\mathrm{s},-}} & =-1+2\alpha_\textrm{n}d_i^{\mathrm{s},-}-2\alpha_\textrm{n}d_i^{\mathrm{t},-}
    \end{aligned}
  \end{equation}

  Accordingly, by letting both
  $\frac{\partial\mathcal{L}_\textrm{Total}}{\partial d_i^{\mathrm{s},+}}$
  and $\frac{\partial\mathcal{L}_\textrm{Total}}{\partial d_i^{\mathrm{s},-}}$
  in \cref{eq:derivatives} to be 0, we obtain:
  \begin{equation}
    \begin{aligned}
      d_i^{\mathrm{s},+} & =d_i^{\mathrm{t},+}-\frac{1}{2\alpha_\textrm{p}} \\
      d_i^{\mathrm{s},-} & =d_i^{\mathrm{t},-}+\frac{1}{2\alpha_\textrm{n}}
    \end{aligned}
    \label{eqn:opt}
  \end{equation}

  Since $\frac{1}{2\alpha_\textrm{p}}$ and $\frac{1}{2\alpha_\textrm{n}}$ are two positive numbers,
  it holds:
  \begin{equation}\nonumber
    d_i^{\mathrm{s},+} <  d_i^{\mathrm{t},+}, \quad d_i^{\mathrm{s},-} > d_i^{\mathrm{t},-}, i = 1,2,...,N
  \end{equation}
\end{proof}

\cref{prop} reveals that given a pre-trained teacher model, the distances
of positive descriptor pairs in the student model which is trained by jointly
minimizing the triplet loss and the designed regularizer are smaller than those
in the teacher model at the training stage in theory, while the distances of
negative descriptor pairs in the student model are larger than those in the
teacher model.

It is also noted that \cref{prop} holds true under the condition that
`$m+d_i^{\mathrm{s},+}-d_i^{\mathrm{s},-}>0$' ($i=1,2,\cdots,N$), and this condition could be met with
a high probability when $m$ is set to be a large constant. Here, we also conduct
the following experiment to empirically analyze this condition:
We train HardNet \citep{HardNet} (a typical local descriptor network which
is trained by minimizing the classic triplet loss where the margin is set to be $m=1$)
on the Liberty subset of Brown and derive the
corresponding equal-weight student model DesDis-HardNet accordingly.
Then, we randomly sample 100K triplets
from the Liberty subset of Brown, and
100K triplets from the HPatches dataset \citep{HPatches} respectively,
by following the hard negative mining strategy used in HardNet.
The corresponding $\{m+d_i^{\mathrm{s},+}-d_i^{\mathrm{s},-}\}_{i=1}^{100000}$ from
the student model DesDis-HardNet are obtained,
and the distributions are shown in \cref{fig:desdis_hardnet_brown} and
\cref{fig:desdis_hardnet_hpatches}.
As seen from the two figure, the condition
`$m+d_i^{\mathrm{s},+}-d_i^{\mathrm{s},-}>0$' ($i=1,2,\cdots,N$) holds
true in most cases on both the Brown and HPatches datasets.
In fact, less than 100 of 100K samples do not meet this condition on
the two datasets, this is to say, this condition could be met with more than 99.9\% probability.
Moreover, we use the typical handcrafted descriptor SIFT \citep{SIFT} as the
teacher and train the student model DesDis-32
on the Liberty subset of Brown under the proposed framework.
The corresponding $\{m+d_i^{\mathrm{s},+}-d_i^{\mathrm{s},-}\}_{i=1}^{100000}$
on Brown and HPatches are shown in \cref{fig:desdis_sift_brown}
and \cref{fig:desdis_sift_hpatches} respectively. As seen from the two figures,
the condition could also be met with a high probability.

\begin{table*}[htbp]
  \begin{center}
    \caption{Comparative evaluation on the Brown dataset \citep{BrownDataset}.
        The numbers in the seven columns on the right are
        false positive rate at 95\% recall. The best results are in \textbf{bold}.
        `DesDis-' denotes
        the derived student model trained under the proposed DesDis framework.
        `$\dagger$' denotes the baseline models that are
        trained without the proposed teacher-student regularizer.
        The throughputs are evaluated on a GTX 1650Ti GPU.
    }
    \label{tab:brown}
    \newcolumntype{C}[1]{>{\centering\let\newline\\\arraybackslash\hspace{0pt}}m{#1}}
    \def\gridwidth{0.05}
    \def\basemark{$^\dagger$}
    \fs
    \begin{tabular}{lcccccccccc}
        \toprule
        {Train}            & \multirow{2}{*}{\centering \#Param.} & Throughputs   & {ND}                      & {YOS}                     & {LIB}                     & {YOS}     & {LIB}    & {ND}     & \multirow{2}{0.07\textwidth}{\centering Mean} \\
        \cmidrule(l{6pt}r{6pt}){4-5} \cmidrule(l{6pt}r{6pt}){6-7} \cmidrule(l{6pt}r{6pt}){8-9}
        {Test}             &                                      & (K patch/sec) & \multicolumn{2}{c}{LIB}   & \multicolumn{2}{c}{ND}    & \multicolumn{2}{c}{YOS}                                                                                     \\
        \midrule
        SIFT               & -                                    & -             & \multicolumn{2}{c}{29.84} & \multicolumn{2}{c}{22.53} & \multicolumn{2}{c}{27.29} & 26.55                                                                           \\
        MatchNet           & -                                    & -             & 7.04                      & 11.47                     & 3.82                      & 5.65      & 11.6     & 8.70     & 8.05                                          \\
        TFeat              & 0.60M                                & 100           & 7.39                      & 10.13                     & 3.06                      & 3.80      & 8.06     & 7.24     & 6.64                                          \\
        L2Net              & 1.33M                                & 17            & 2.36                      & 4.70                      & 0.72                      & 1.29      & 2.57     & 1.71     & 2.22                                          \\
        DOAP               & 1.33M                                & 17            & 1.54                      & 2.62                      & 0.43                      & 0.87      & 2.00     & 1.21     & 1.45                                          \\
        \midrule
        \multicolumn{10}{c}{Equal-Weight}                                                                                                                                                                                                               \\
        \midrule
        HardNet            & 1.33M                                & 17            & 1.49                      & 2.51                      & 0.53                      & 0.78      & 1.96     & 1.84     & 1.51                                          \\
        DesDis-HardNet     & 1.33M                                & 17            & \bf 1.31                  & \bf 2.03                  & \bf 0.48                  & \bf  0.72 & \bf 1.54 & \bf 1.31 & \bf 1.23                                      \\
        \cdashline{1-10}[3pt/3pt]\noalign{\smallskip}
        SOSNet             & 1.33M                                & 17            & 1.08                      & 2.12                      & \bf 0.35                  & 0.67      & 1.03     & 0.95     & 1.03                                          \\
        DesDis-SOSNet      & 1.33M                                & 17            & \bf 1.02                  & \bf 1.91                  & 0.36                      & \bf 0.60  & \bf 0.96 & \bf 0.82 & \bf 0.95                                      \\
        \cdashline{1-10}[3pt/3pt]\noalign{\smallskip}
        HyNet              & 1.34M                                & 12            & 0.89                      & \bf 1.37                  & 0.34                      & 0.61      & 0.88     & 0.96     & 0.84                                          \\
        DesDis-HyNet       & 1.34M                                & 12            & \bf 0.86                  & \bf 1.37                  & \bf 0.29                  & \bf 0.48  & \bf 0.70 & \bf 0.60 & \bf 0.71                                      \\
        \midrule
        \multicolumn{10}{c}{Light-Weight}                                                                                                                                                                                                               \\
        \midrule
        DesDis-8\basemark  & 0.08M                                & 440           & 4.06                      & 5.61                      & 1.69                      & 1.88      & 4.96     & 3.74     & 3.66                                          \\
        DesDis-8           & 0.08M                                & 440           & \bf 3.95                  & \bf 5.23                  & \bf 1.46                  & \bf 1.68  & \bf 4.69 & \bf 3.71 & \bf 3.45                                      \\
        \cdashline{1-10}[3pt/3pt]\noalign{\smallskip}
        DesDis-16\basemark & 0.19M                                & 294           & \bf 2.35                  & 3.77                      & 0.77                      & 1.13      & 2.97     & 2.18     & 2.20                                          \\
        DesDis-16          & 0.19M                                & 294           & 2.37                      & \bf 3.15                  & \bf 0.72                  & \bf 1.05  & \bf 2.59 & \bf 2.06 & \bf 1.99                                      \\
        \cdashline{1-10}[3pt/3pt]\noalign{\smallskip}
        DesDis-24\basemark & 0.33M                                & 183           & 2.04                      & 2.96                      & 0.63                      & 0.96      & 2.06     & 1.73     & 1.73                                          \\
        DesDis-24          & 0.33M                                & 183           & \bf 1.91                  & \bf 2.69                  & \bf 0.59                  & \bf 0.88  & \bf 1.87 & \bf 1.69 & \bf 1.61                                      \\
        \cdashline{1-10}[3pt/3pt]\noalign{\smallskip}
        DesDis-32\basemark & 0.50M                                & 145           & 1.81                      & 2.86                      & 0.55                      & 0.90      & 2.01     & 1.69     & 1.64                                          \\
        DesDis-32          & 0.50M                                & 145           & \bf 1.52                  & \bf 2.36                  & \bf 0.54                  & \bf 0.81  & \bf 1.68 & \bf 1.48 & \bf 1.39                                      \\
        \bottomrule[1.0pt]
    \end{tabular}
  \vspace{-0.1in}
\end{center}
\end{table*}

\begin{figure*}[htbp]
  \centering
  \input{hpatches.tex}
\end{figure*}

\section{Experimental Results}\label{sec:experiments}
In this section, we first describe the used three public datasets and 
the comparison group.
Then, we introduce the implementation details.
Next, we evaluate both the equal-weight and light-weight
student models derived from the proposed
DesDis framework on the three public datasets.
Finally, the ablation study is provided.

\subsection{Datasets and Comparison Group}\label{subsec:setup}

\noindent\textbf{Datasets: }
In this paper, we use three public datasets for evaluating our method, including
the Brown dataset \citep{BrownDataset}, the HPatches dataset \citep{HPatches}
and the ETH SfM dataset \citep{ETH}.

The Brown dataset \citep{BrownDataset} is the most widely used patch
dataset for evaluating the patch verification performance of local descriptors.
It consists of three subsets: {\it Liberty}, {\it Notredame} and {\it Yosemite}.
The test set consists of 100K matching and non-matching pairs for each sequence.

The HPatches dataset \citep{HPatches} consists of over 1.5 million patches extracted from
116 viewpoint and illumination changing scenes.
According to the geometric noise levels, the extracted patches are classified into
the {\it easy}, {\it hard}, and {\it tough} groups respectively. Three tasks are performed
on the three groups of patches, including  patch verification, image matching, and patch
retrieval.

The ETH benchmark \citep{ETH} is designed for image-based 3D reconstruction.
We follow the setup in \citep{SOSNet,HyNet}, \textit{i.e.},
all learning-based methods are trained on the {\it Liberty}
subset of Brown \citep{BrownDataset}. Since the
patches used in the comparative methods are not given in the original papers
\citep{SOSNet,HyNet}, we extract the patches using the DoG detector, and use these
patches for all the methods.

\noindent\textbf{Comparison Group: }
As indicated in \cref{sec:introduction}, under the proposed DesDis framework,
not only a more discriminative equal-weight student model,
but also a set of light-weight student models that achieve a trade-off between
computational accuracy and speed could be derived.
Accordingly, the comparative evaluation is divided into two groups:
In \cref{subsec:equalweight}, we evaluate the equal-weight student models
derived under the proposed DesDis framework.
In \cref{subsec:lightweight}, we evaluate the light-weight student models.

\subsection{Implementation Details}\label{subsec:details}
In the equal-weight student model scenario, we firstly train the teacher models by
implementing the code released by \citet{SOSNet,HyNet}
for 200 epochs with the learning rate of 0.01 and the batch size of 1024.
We use Adam optimizer with $\alpha=0.01$, $\beta_1=0.9$
and $\beta_2=0.999$.
Then, the trained teacher models are used for deriving
the corresponding student models under the proposed framework.
We set the weights $\alpha_\textrm{p}$ and $\alpha_\textrm{n}$ in \cref{eq:Lobj}
to 1 and 15 respectively.

In the light-weight student model scenario, both $\alpha_\textrm{p}$ and $\alpha_\textrm{n}$
are set to 9. The models are trained for 200 epochs with the learning rate of
0.01 and Adam optimizer. We adopt HyNet \citep{HyNet} as the teacher model for
all the light-weight student models.

\subsection{Equal-Weight Student Models Under DesDis}\label{subsec:equalweight}

In this subsection, we use three typical local descriptor learning networks
HardNet \citep{HardNet}, SOSNet \citep{SOSNet}, HyNet \citep{HyNet} as the teacher
models under the proposed DesDis framework, and derive the corresponding equal-weight
students, denoted as DesDis-HardNet, DesDis-SOSNet, DesDis-HyNet.

\subsubsection{Evaluation on Brown}
We evaluate the
derived three equal-weight student models in comparison to their teacher models and
6 additional state-of-the-arts methods, including SIFT \citep{SIFT},
DeepDesc \citep{DeepDesc}, MatchNet \citep{MatchNet}, TFeat \citep{TNet}, L2Net \citep{L2Net}
and DOAP \citep{DOAP}. As done in \citep{TNet,HardNet,SOSNet,HyNet},
all the learning-based methods are trained on one subset, and then tested on the other two.

The false positive rates at 95\% recall by all the comparative methods are reported in
\cref{tab:brown}. As seen from this table, the performances of the three derived
equal-weight student models are better than their corresponding teachers, and the derived
DesDis-HyNet performs best among all the comparative methods. These results
demonstrate that the proposed DesDis framework with the designed TS regularizer
could effectively boost the performances of the existing networks.


\subsubsection{Evaluation on HPatches}

We evaluate the three derived equal-weight student models in comparison to
their teacher models HardNet \citep{HardNet}, SOSNet \citep{SOSNet}, HyNet \citep{HyNet},
the relatively shallow network TFeat \citep{TNet} and the handcrafted descriptor
SIFT \citep{SIFT}.

As done in previous works \citep{HardNet,DOAP,SOSNet,HyNet}, we train the models on the
  {\it Liberty} subset of Brown, and then evaluate on the test split `a' of HPatches.
The corresponding mean Average Precision (mAP) of each method are reported in
\cref{fig:hpatches_a}. As seen from this figure, the derived student models
outperform their teachers respectively, and the derived DesDis-HyNet (DesDis-Hy)
performs best among all the comparative methods, which are consistent with those
on the Brown dataset as reported above.

\subsubsection{Evaluation on ETH}

\begin{table*}[htbp]
  \caption{Evaluation results on the ETH benchmark \citep{ETH} for SfM tasks.
    `$\dagger$' denotes the baseline models that are
    trained without the proposed teacher-student regularizer.
    The best results among equal-weight models are marked in \textcolor{red}{red}.
    The best results among our light-weight models are marked in \textcolor{blue}{blue}.
  }
  \label{tab:ETH}
  \begin{center}
    \newcolumntype{C}[1]{>{\centering\let\newline\\\arraybackslash\hspace{0pt}}m{#1}}
    \newcolumntype{L}[1]{>{\raggedright\let\newline\\\arraybackslash\hspace{0pt}}m{#1}}
    \newcommand{\red}[1]{\textcolor[rgb]{1,0.3,0.3}{#1}}
    \newcommand{\kd}[1]{\textcolor[rgb]{0.3,0.3,1}{#1}}
    \def\gridwidth{0.1}
    \def\basemark{$^\dagger$}
    \fs
    \begin{tabular}{C{0.13\textwidth}L{0.15\textwidth}C{0.08\textwidth}C{\gridwidth\textwidth}C{0.13\textwidth}C{\gridwidth\textwidth}C{0.09\textwidth}}
      \toprule
                      &                    & \bf \#Reg. Images & \bf \#Sparse Points & \bf \#Observations & \bf Track Length & \bf Reproj. Error \\
      \midrule
      \bf Herzjesu    & SIFT               & 8                 & 7.7K                & 30K                & 3.90             & {0.39px}          \\
      \bf 8 images    & TFeat              & 8                 & 7.7K                & 30K                & 3.90             & 0.41px            \\
                      & HardNet            & 8                 & 8.8K                & 36K                & 4.04             & 0.44px            \\
                      & DesDis-HardNet     & 8                 & 8.7K                & 36K                & 4.05             & 0.44px            \\
                      & SOSNet             & 8                 & 8.8K                & 36K                & {4.06}           & 0.44px            \\
                      & DesDis-SOSNet      & 8                 & 8.8K                & 36K                & {4.06}           & 0.44px            \\
                      & HyNet              & 8                 & 8.9K                & 36K                & 4.05             & 0.45px            \\
                      & DesDis-HyNet       & 8                 & \red{9.1K}          & \red{37K}          & 4.05             & 0.45px            \\
      \cdashline{2-7}[3pt/3pt]\noalign{\smallskip}
                      & DesDis-8\basemark  & 8                 & 8.4K                & 33.8K              & 4.03             & {0.42px}          \\
                      & DesDis-8           & 8                 & 8.4K                & 34.0K              & 4.03             & {0.42px}          \\
                      & DesDis-32\basemark & 8                 & 8.6K                & 34.7K              & {4.04}           & 0.45px            \\
                      & DesDis-32          & 8                 & \kd{8.7K}           & \kd{35.2K}         & {4.04}           & 0.43px            \\
      \midrule
      \bf Fountain    & SIFT               & 11                & 16.6K               & 76K                & 4.57             & {0.34px}          \\
      \bf 11 images   & TFeat              & 11                & 16.4K               & 75K                & 4.55             & 0.35px            \\
                      & HardNet            & 11                & 17.7K               & 83K                & 4.69             & 0.38px            \\
                      & DesDis-HardNet     & 11                & 17.7K               & 83K                & 4.69             & 0.38px            \\
                      & SOSNet             & 11                & 17.7K               & 83K                & {4.71}           & 0.38px            \\
                      & DesDis-SOSNet      & 11                & 17.7K               & 84K                & {4.71}           & 0.38px            \\
                      & HyNet              & 11                & 17.9K               & 84K                & 4.70             & 0.39px            \\
                      & DesDis-HyNet       & 11                & \red{18.0K}         & \red{85K}          & {4.71}           & 0.39px            \\
      \cdashline{2-7}[3pt/3pt]\noalign{\smallskip}
                      & DesDis-8\basemark  & 11                & 17.3K               & 81K                & 4.65             & {0.36px}          \\
                      & DesDis-8           & 11                & 17.4K               & 81K                & 4.66             & 0.37px            \\
                      & DesDis-32\basemark & 11                & 17.5K               & 82K                & {4.68}           & 0.38px            \\
                      & DesDis-32          & 11                & \kd{17.7K}          & \kd{83K}           & 4.67             & 0.38px            \\
      \midrule
      \bf Madrid      & SIFT               & 424               & 93K                 & 562K               & 6.02             & {0.59px}          \\
      \bf Metropolis  & TFeat              & 402               & 81K                 & 479K               & 5.89             & 0.61px            \\
      \bf 1344 images & HardNet            & 460               & 129K                & 778K               & 6.05             & 0.66px            \\
                      & DesDis-HardNet     & 478               & 134K                & 799K               & 5.96             & 0.68px            \\
                      & SOSNet             & 472               & 130K                & 793K               & {6.09}           & 0.68px            \\
                      & DesDis-SOSNet      & 481               & 137K                & 810K               & 5.93             & 0.67px            \\
                      & HyNet              & 478               & 141K                & 840K               & 5.94             & 0.68px            \\
                      & DesDis-HyNet       & \red{486}         & \red{145K}          & \red{847K}         & 5.86             & 0.69px            \\
      \cdashline{2-7}[3pt/3pt]\noalign{\smallskip}
                      & DesDis-8\basemark  & 436               & 107K                & 671K               & {6.25}           & {0.64px}          \\
                      & DesDis-8           & 448               & 117K                & 704K               & 6.02             & {0.64px}          \\
                      & DesDis-32\basemark & 449               & 125K                & 747K               & 5.96             & 0.66px            \\
                      & DesDis-32          & \kd{455}          & \kd{127K}           & \kd{769K}          & 6.04             & 0.67px            \\
      \midrule
      \bf Gendar-     & SIFT               & 936               & 304K                & 1449K              & 4.76             & {0.72px}          \\
      \bf menmarkt    & TFeat              & 900               & 274K                & 1266K              & 4.61             & 0.75px            \\
      \bf 1463 images & HardNet            & 993               & 368K                & 1900K              & 5.16             & 0.78px            \\
                      & DesDis-HardNet     & 994               & 373K                & 1919K              & 5.14             & 0.79px            \\
                      & SOSNet             & 986               & 371K                & 1917K              & {5.17}           & 0.78px            \\
                      & DesDis-SOSNet      & \red{1021}        & \red{401K}          & \red{2042K}        & 5.09             & 0.77px            \\
                      & HyNet              & 982               & 386K                & 1993K              & {5.17}           & 0.79px            \\
                      & DesDis-HyNet       & 994               & 389K                & 2016K              & {5.17}           & 0.79px            \\
      \cdashline{2-7}[3pt/3pt]\noalign{\smallskip}
                      & DesDis-8\basemark  & 971               & 339K                & 1726K              & 5.10             & {0.76px}          \\
                      & DesDis-8           & 977               & 338K                & 1732K              & {5.12}           & {0.76px}          \\
                      & DesDis-32\basemark & 972               & 362K                & 1846K              & 5.10             & 0.77px            \\
                      & DesDis-32          & \kd{1017}         & \kd{400K}           & \kd{2019K}         & 5.05             & 0.77px            \\
      \bottomrule[1.0pt]
    \end{tabular}
  \end{center}
\end{table*}

The comparative evaluation results on the ETH benchmark are reported
above the dashed lines in \cref{tab:ETH}.
As seen from this table, SIFT obtains the smallest re-projection error, and a
method that obtains a smaller number of matching points is prone to obtaining a lower
re-projection error, which is consistent with the observations in \citep{GeoDesc,SOSNet,HyNet}.
However, since the calculated re-projection errors by all the comparative methods
are smaller than 1 pixel, this metric might not sufficiently reflect the performance
differences among different descriptors as indicated in \citep{SOSNet}. It is further
noted that the derived equal-weight student models achieve close performances to their
teachers on the two small-scale high-quality sequences {\it Herzjesu} and {\it Fountain}.
But for the two large-scale sequences ({\it Madrid Metropolis} and {\it Gendarmenmarkt})
with high noises, compared with their teachers, the derived equal-weight student models
obtain close mean track lengths and achieve significantly better performances under the three
protocols (the number of the registered images, the number of the reconstructed sparse
points and the number of observations) in most cases, showing that the derived models
under DesDis are more effective for handling large-scale noisy data.

\subsection{Light-Weight Student Models Under DesDis}\label{subsec:lightweight}

In this subsection, we evalue the efficiency of the proposed framework for deriving 
light-weight student models.
As done in the above evaluation on the equal-weight models, the derived light-weight
models are also evaluated on Brown, HPatches and ETH, in comparison to several
state-of-the-arts methods. Moreover, we evaluate the architectures
of these light-weight models trained with only the loss term from HyNet
\citep{HyNet} as baselines for further demonstrating the effectiveness of the proposed
TS regularizer, denoted as DesDis-8$^\dagger$, DesDis-16$^\dagger$,
DesDis-24$^\dagger$, DesDis-32$^\dagger$ respectively.

\begin{figure*}[t!]
  \def\scale{0.22}
  \def\inv{0.2in}
  \centering
  \subfloat[]{
    \includegraphics[width=\scale\textwidth]{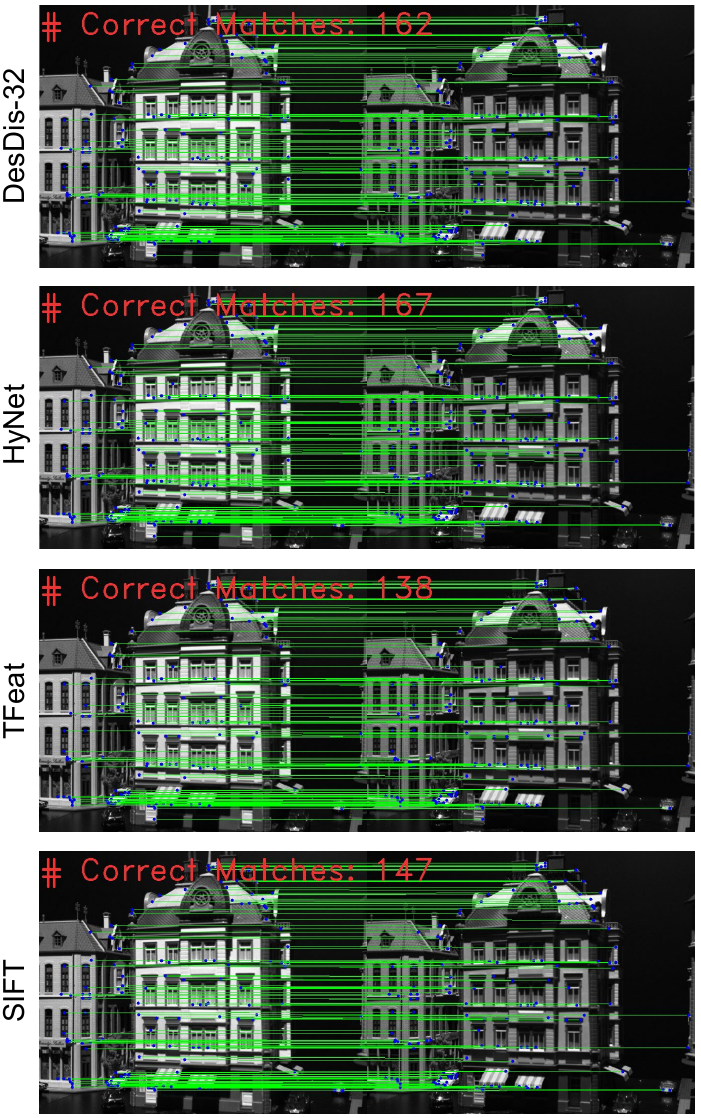}
    \label{fig:yello}
  }
  \hspace{\inv}
  \subfloat[]{
    \includegraphics[width=\scale\textwidth]{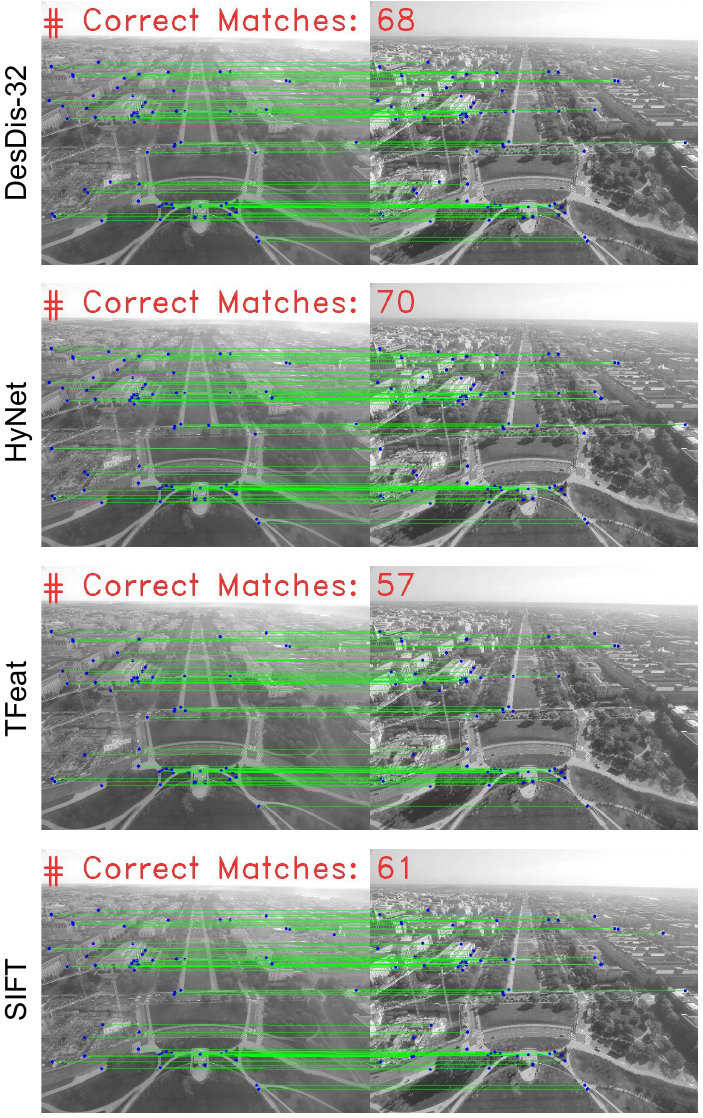}
    \label{fig:dc}
  }
  \hspace{\inv}
  \subfloat[]{
    \includegraphics[width=\scale\textwidth]{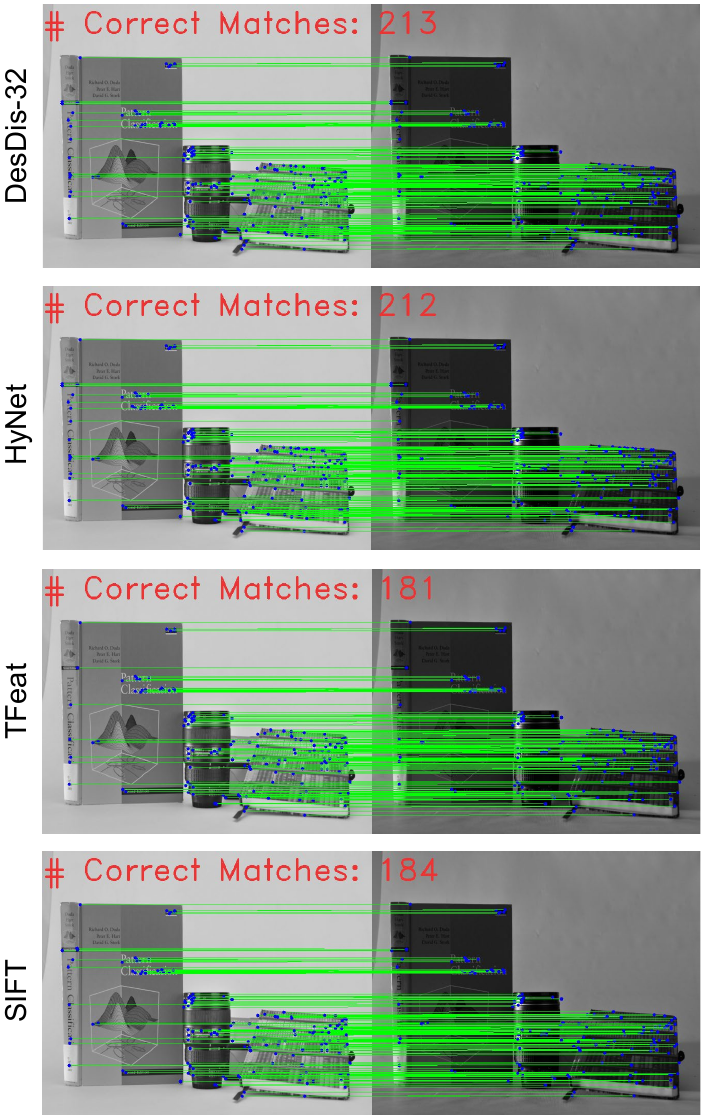}
    \label{fig:duda}
  }\\
  \subfloat[]{
    \includegraphics[width=\scale\textwidth]{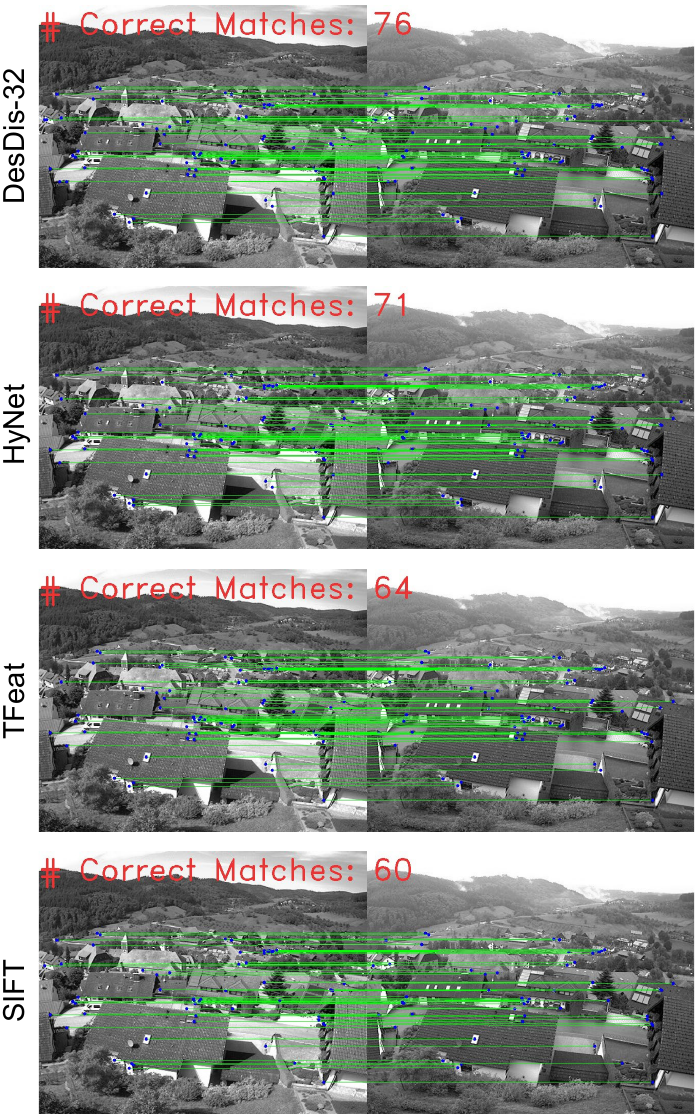}
    \label{fig:village}
  }
  \hspace{\inv}
  \subfloat[]{
    \includegraphics[width=\scale\textwidth]{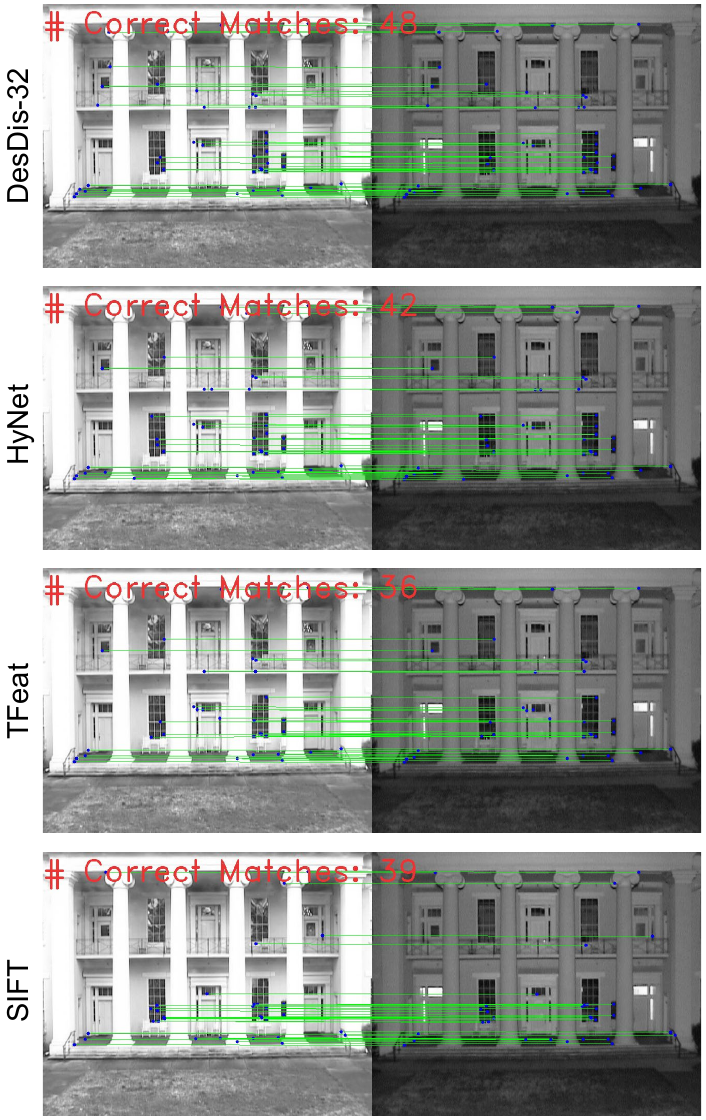}
    \label{fig:kion}
  }
  \hspace{\inv}
  \subfloat[]{
    \includegraphics[width=\scale\textwidth]{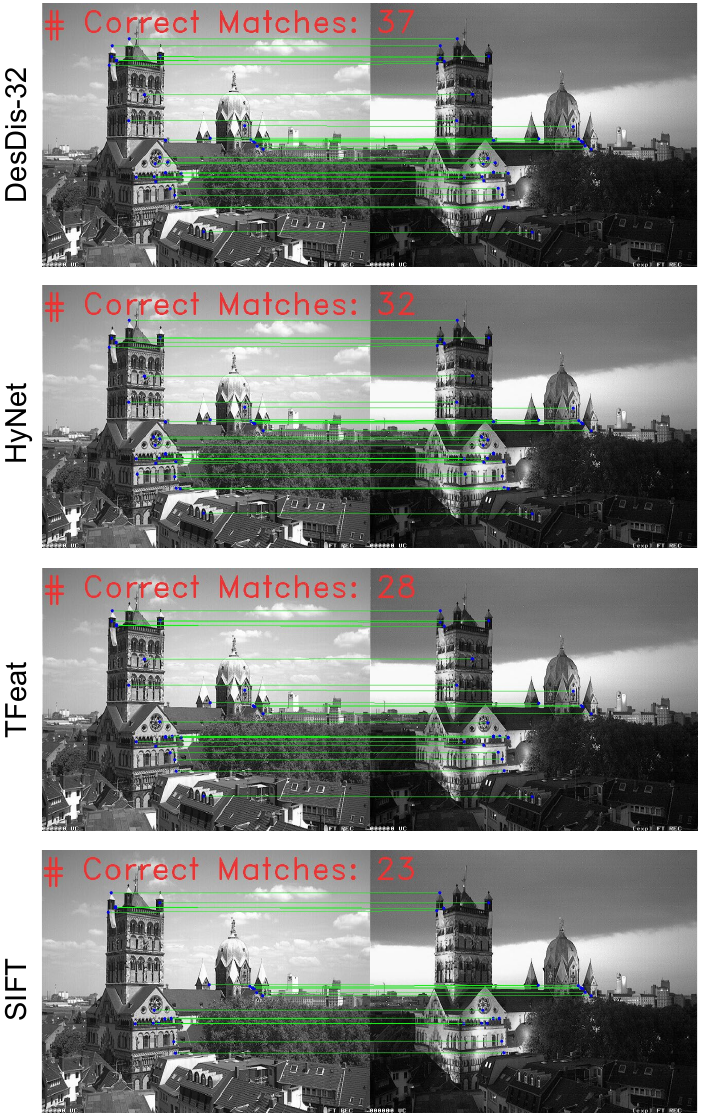}
    \label{fig:castle}
  }
  \caption{Image matching visualization on HPatches by the proposed DesDis-32,
    HyNet, TFeat and SIFT. The models are trained on the Liberty subset of Brown.
    The number of correct matches is shown in the upper left corner of each image pair.}
  \vspace{-0.1in}
\end{figure*}

\begin{figure*}
  \def\scale{1}
  \subfloat[Examples where 
    the teacher model succeeds (solid lines) but the student model fails (dashed lines).]{
    \includegraphics[width=\scale\textwidth]{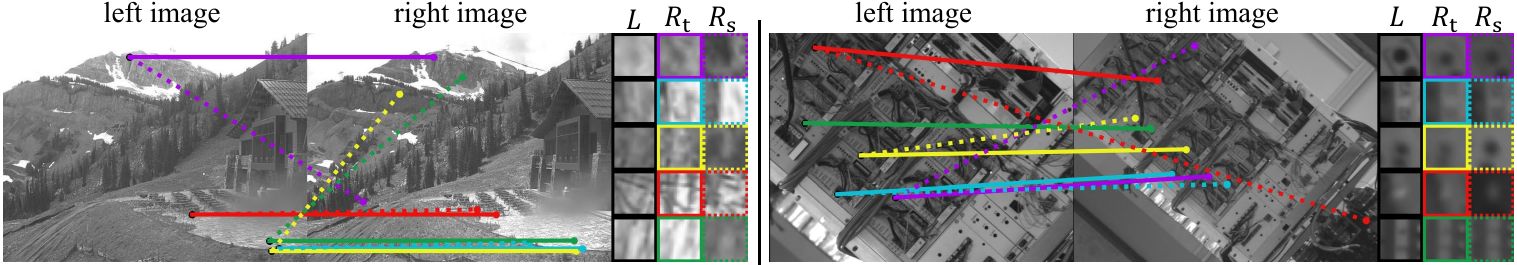}
    \label{fig:tea}
  }\\
  \subfloat[Examples where   the student model succeeds (solid lines) 
    but the teacher model fails (dashed lines).]{
    \includegraphics[width=\scale\textwidth]{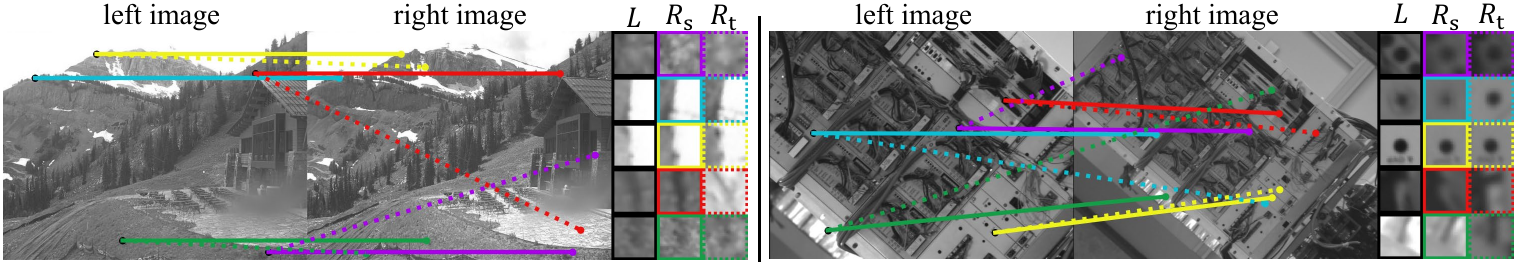}
    \label{fig:stu}
  }
\caption{
Visualization of the keypoint matches by the teacher HyNet and the corresponding 
student DesDis-32 on HPatches.  
The left part of each group shows a pair of images, where five point matches by 
both the teacher and student are connected with colorful lines. 
The right part of each group shows the corresponding local patches to the five keypoints, 
and in each row, $L$ corresponds to a keypoint in the left image, 
$R_\mathrm{t}$ corresponds to the matched keypoint by the teacher in the right image, 
and $R_\mathrm{s}$ corresponds to the matched keypoint by the student in the right image.
}
\label{fig:corner}
\end{figure*}

\begin{table}[t]
  \centering
  \caption{Comparison of inference speeds (average number of images processed per second)
    on the HPatches dataset with 2048 keypoints. All the methods
    are evaluated on a GTX 1650Ti GPU with a batch size of 1024.
  }
  \label{tab:hp_speed}
  \fs
  \begin{tabular}{cc}
    \toprule
    Method         & Throughputs (images/sec) \\
    \midrule
    TFeat          & 72                       \\
    L2Nte          & 14                       \\
    DOAP           & 14                       \\
    HardNet/       & \multirow{2}{*}{14}      \\
    DesDis-HardNet &                          \\
    SOSNet/        & \multirow{2}{*}{14}      \\
    DesDis-SOSNet  &                          \\
    HyNet/         & \multirow{2}{*}{9}       \\
    DesDis-HyNet   &                          \\
    \midrule
    DesDis-8       & 239                      \\
    DesDis-16      & 210                      \\
    DesDis-24      & 126                      \\
    DesDis-32      & 99                       \\
    \bottomrule
  \end{tabular}
  \vspace{-0.1in}
\end{table}

\subsubsection{Evaluation on Brown}
The false positive rates at 95\% recall by these light-weight models on
the Brown dataset \citep{BrownDataset} are reported in \cref{tab:brown}.
In addition, as done in \citep{L2Net,TNet}, the numbers of patches processed
per second by the comparative methods are also reported in \cref{tab:brown} for
comparing their inference speeds.
The following three points are observed from this table:
(i) All the light-weight student models surpass their baselines that are trained
without the proposed TS regularizer, demonstrating the effectiveness of the
designed TS regularizer for improving the performances of light-weight models;
(ii) Compared with SOSNet \citep{SOSNet} and HyNet \citep{HyNet}, the light-weight
student models achieve worse performances, but smaller computational memories and
faster speeds. This is mainly because the light-weight student models need much fewer
parameters and much simpler architectures as in \cref{tab:brown}, their ability
of feature representation is weakened accordingly, but their computational
costs are significantly decreased;
(iii) Compared with TFeat \citep{TNet} which is a relatively fast descriptor in the existing
works, our smallest student model DesDis-8 achieves a significantly better performance with
a much faster speed. With slightly or significantly better performances, our light-weight
models achieve 8 times or even faster speeds for processing patches
(e.g. DesDis-16 vs L2Net: $\approx$17.3 times faster;
DesDis-32 vs HardNet/DOAP: $\approx$8.5 times faster).

\subsubsection{Evaluation on HPatches}
The comparative results of the light-weight models with/without the TS regularizer
on the HPatches dataset \citep{HPatches} are
reported in \cref{fig:hpatches_b}.
As seen from this figure, the derived models
(DesDis-8, DesDis-16, DesDis-24, DesDis-32) with the TS regularizer consistently
outperform their counterparts without it. These results further
demonstrate the effectiveness of the designed TS regularizer.
In addition, in order to further compare the inference speeds of the
comparative methods for images (rather than patches as reported in \cref{tab:brown}),
\cref{tab:hp_speed} reports the average number of images processed
per second on the HPatches benchmark,
including 6 existing methods  (TFeat, L2Net, DOAP, HardNet, SOSNet, HyNet),
the proposed 3 equal-weight student models (DesDis-HardNet, DesDis-SOSNet, DesDis-HyNet),
and the proposed 4 light-weight student models (DesDis-8, DesDis-16, DesDis-24, DesDis-32).
All the methods use a fixed number (2048) of keypoints, and are evaluated on
a 1650Ti GPU with a batch size of 1024.
As seen from this table, our light-weight student models still have significantly faster
inference speeds than the other comparative methods.

\subsubsection{Evaluation on ETH}
Considering that it is much time-consuming to perform the structure from motion task on
large-scale scenes in the ETH dataset \citep{ETH}, we only evaluate the smallest variant
DesDis-8 and the largest variant DesDis-32 among the derived four light-weight models,
as well as their baseline models DesDis-8$^\dagger$ and DesDis-32$^\dagger$ respectively.
The results are reported in \cref{tab:ETH}. As noted from this table, in terms of
the numbers of registered images, reconstructed sparse points and observations,
DesDis-8 and DesDis-32 perform better than SIFT \citep{SIFT}, TFeat \citep{TNet} and their
corresponding baselines.  It's also worth noting that in terms of the aforementioned 3
metrics, our DesDis-32 can surpass HardNet, SOSNet and HyNet on the largest
scene {\it Gendarmenmarkt}, and it achieves a close performance to HardNet on other scenes
with an 8.5 times faster speed. All the results demonstrate that the derived light-weight
models from the proposed DesDis framework could achieve an effective
trade-off between computational accuracy and speed.

\begin{figure*}[t]
  \centering
  \includegraphics[width=0.85\textwidth]{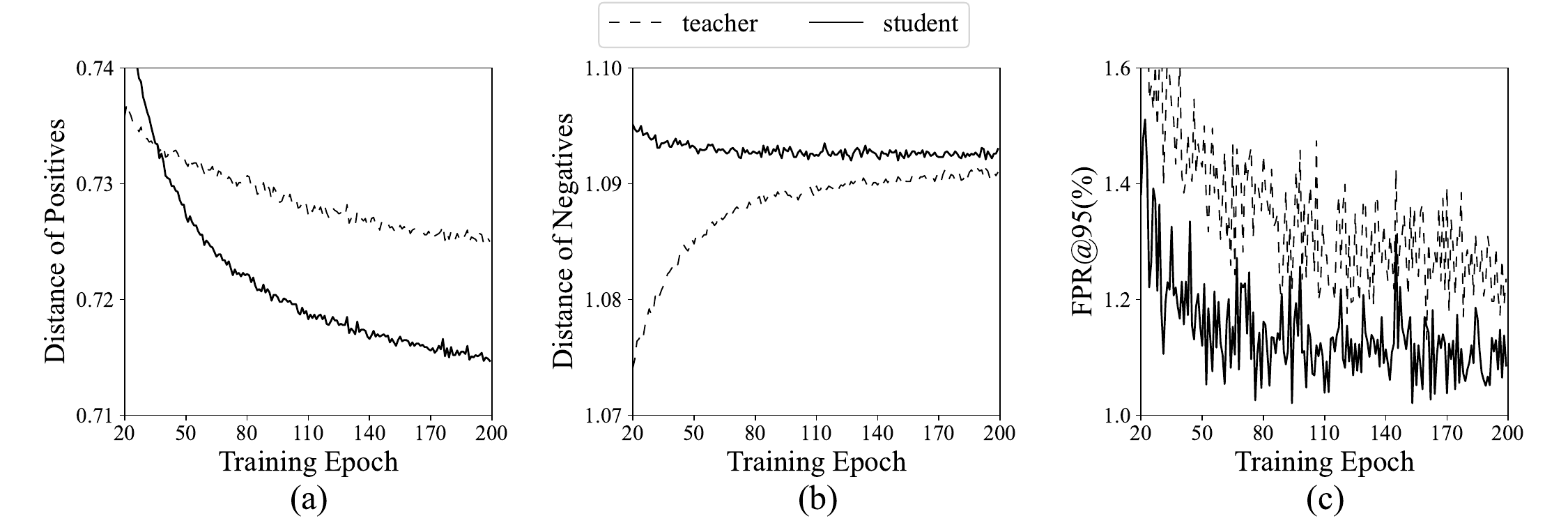}
  \vspace{-0.06in}
  \caption{
    Comparison of positive/negative pair distance and the model
    performance during training. The model is trained on
    the Liberty subset of the Brown dataset, and then tested on the Yosemite
    and Notredame subsets.
    (a) The distance of positive pairs during the training procedure.
    (b) The distance of negative pairs during the training procedure.
    (c) The false positive rate at 95\% recall during training.
  }
  \label{fig:distance}
  \vspace{-0.1in}
\end{figure*}

\subsubsection{Visualization Results}

We provide several image matching visualization results from the HPatches
dataset \citep{HPatches} by the proposed light-weight model DesDis-32 and
three typical descriptors, including the DesDis-32's teacher model HyNet
\citep{HyNet}, TFeat \citep{TNet} (a relatively fast descriptor in literature)
and the hand-crafted SIFT \citep{SIFT}. The keypoints are all detected by the DoG
detector. As seen from
Fig. \ref{fig:yello}-\ref{fig:castle}, DesDis-32 and HyNet are able to obtain
a larger number of correct correspondences than SIFT and TFeat. In most
cases, DesDis-32 could achieve a close number of correct matches to HyNet
(Fig. \ref{fig:yello}-\ref{fig:duda}), while in some challenging scenes,
DesDis-32 could even produce more correct matches
(Fig. \ref{fig:village}-\ref{fig:castle}). It has to be pointed out that
the complex architecture of HyNet results in a small throughput (12K
patch/sec on a GTX 1650Ti), and our light-weight model DesDis-32 runs at 145K patch/sec
($12\times$ faster) while maintaining a comparable image matching accuracy.

\cref{fig:corner} further compares the keypoint matching results by the 
teacher HyNet and its student DesDis-32 respectively 
(including wrongly/correctly matched keypoint and their 
corresponding local patches) on HPatches.
\cref{fig:tea} shows the cases that succeed in the teacher model 
but fail in the student model. 
\cref{fig:stu} shows the cases that succeed in the student model but fail in the teacher model.
It is found from these two figures
that local patches with similar textures are challenging for both the teacher 
and student models, in other words, for a same pair of images, the teacher 
succeeds in discriminating some pairs of local patches with similar textures 
whereas the student fails, and the teacher also fails in discriminating some 
other pairs of local patches with similar textures whereas the student succeeds.

\subsection{Ablation Study}

\subsubsection{Impact of the Proposed TS Regularizer}\label{subsubsec:dist}
As analyzed in \cref{sec:analysis}, the TS regularizer could help the student
model predict a smaller distance of positive pairs and a larger distance for
negative pairs than its teacher network, leading to a better performance. To verify
this theoretical analysis, we firstly record the mean distance of
positive pairs (also negative pairs)  by the teacher and student models
on the Liberty subset of the Brown dataset \citep{BrownDataset} at each training epoch.
Accordingly, we visualize the corresponding curves of the mean distance of
positive pairs in \cref{fig:distance}\textcolor{blue}{a},
as well as  the curves of the mean distance of negative pairs  in
\cref{fig:distance}\textcolor{blue}{b}.
As seen from the two figures, when the number of epochs becomes larger than around 40,
the student could predict a smaller distance for positive pairs than its teacher,
and a larger distance for negative pairs.
These results demonstrate the aforementioned deduction from \cref{prop}
to some extent. Secondly, for each trained model corresponding to
each training epoch on the Liberty subset, we test it on the other two subsets
of the Brown dataset and record its mean false positive rate at 95\% recall.
Accordingly, we visualize the corresponding false positive rate curves of
the teacher and student models in \cref{fig:distance}\textcolor{blue}{c}.
As seen from this figure, the student model could outperform its teacher.
It is also noted that the distance differences between the student and teacher model
illustrated in \cref{fig:distance}\textcolor{blue}{a}
and \cref{fig:distance}\textcolor{blue}{b} is smaller than the optimal
solution indicated by \cref{eqn:opt}, indicating the student model could not find the
global optimal solution. However,
it should be pointed out that \cref{prop} is not intended to prove whether a
student model could obtain a globally optimal solution or not, but to provide
a theoretical basis (regardless of whether a globally optimal solution could
be obtained or not) that the distances of positive pairs in the student
model would become smaller than those in the teacher
model at the training stage, while the distances of negative pairs in the student
model would become larger than those in the teacher model.

\begin{table}[t]
  \centering
  \caption{
    Results by setting the distance supervision $d_\mathrm{p}$ for positive
    pairs to \{0.5, 0.6, 0.7\} and the distance supervision $d_\mathrm{n}$ for
    negative pairs to \{1.0, 1.1, 1.2\}.
    The numbers in the three columns on the right are
    the false positive rate at 95\% recall.
  }
  \label{tab:fix}
  \fs
  \begin{tabular}{cccc}
    \toprule
    Method                               & YOS$\downarrow$ & ND$\downarrow$ & Mean$\downarrow$ \\
    \midrule
    HyNet (baseline)                     & 0.88            & 0.34           & 0.61             \\
    \midrule
    $d_\mathrm{p}=0.5, d_\mathrm{n}=1.0$ & 1.40            & 0.35           & 0.88             \\
    $d_\mathrm{p}=0.5, d_\mathrm{n}=1.1$ & 1.18            & 0.35           & 0.76             \\
    $d_\mathrm{p}=0.5, d_\mathrm{n}=1.2$ & 1.11            & 0.31           & 0.71             \\
    $d_\mathrm{p}=0.6, d_\mathrm{n}=1.0$ & 1.10            & 0.36           & 0.73             \\
    $d_\mathrm{p}=0.6, d_\mathrm{n}=1.1$ & 1.22            & 0.32           & 0.77             \\
    $d_\mathrm{p}=0.6, d_\mathrm{n}=1.2$ & 0.92            & 0.34           & 0.63             \\
    $d_\mathrm{p}=0.7, d_\mathrm{n}=1.0$ & 1.05            & 0.35           & 0.70             \\
    $d_\mathrm{p}=0.7, d_\mathrm{n}=1.1$ & 0.94            & 0.39           & 0.67             \\
    $d_\mathrm{p}=0.7, d_\mathrm{n}=1.2$ & 0.99            & 0.37           & 0.68             \\
    \midrule
    \textbf{DesDis-HyNet}                & \textbf{0.70}   & \textbf{0.29}  & \textbf{0.50}    \\
    \bottomrule
  \end{tabular}
\end{table}

To further evaluate the importance of the proposed TS regularizer,
we conduct the following experiment by replacing the teacher supervisions in
the TS regularizer with fixed distance supervisions,
i.e.,  we replace the positive distances predicted by the teacher model
in \cref{eq:LTSP} with a fixed low distance $d_\mathrm{p}$, and
replace the negative distances in \cref{eq:LTSN}
with a fixed high distance $d_\mathrm{n}$:
Firstly, we train HyNet \citep{HyNet} on the \textit{Liberty} subset of Brown,
and obtain the mean distance of descriptor pairs (around 0.7 for positive pairs and
1.0 for negative pairs) on the \textit{Liberty} subset by HyNet.
Then, considering that the mean distance of positive/negative
pairs of HyNet is approximately 0.7/1.0, we train DesDis-HyNet by setting
$d_\mathrm{p}$ for positive pairs to $\{0.5,0.6,0.7\}$,
and $d_\mathrm{n}$ for negative pairs to $\{1.0,1.1,1.2\}$.
The corresponding results are shown in \cref{tab:fix}.
As seen from this table, such fixed distances does not work as well as
using a teacher model under the proposed framework.
It could be further noted that the results (particularly on the \textit{Yosemite} subset)
by fixed distance supervisions are even poorer than those by the baseline HyNet in
most cases, indicating that fixed distance supervisions might not be helpful
for improving model performances.

\begin{figure}[t]
  \centering
  \subfloat[DesDis-32]{
    \includegraphics[width=0.4\textwidth]{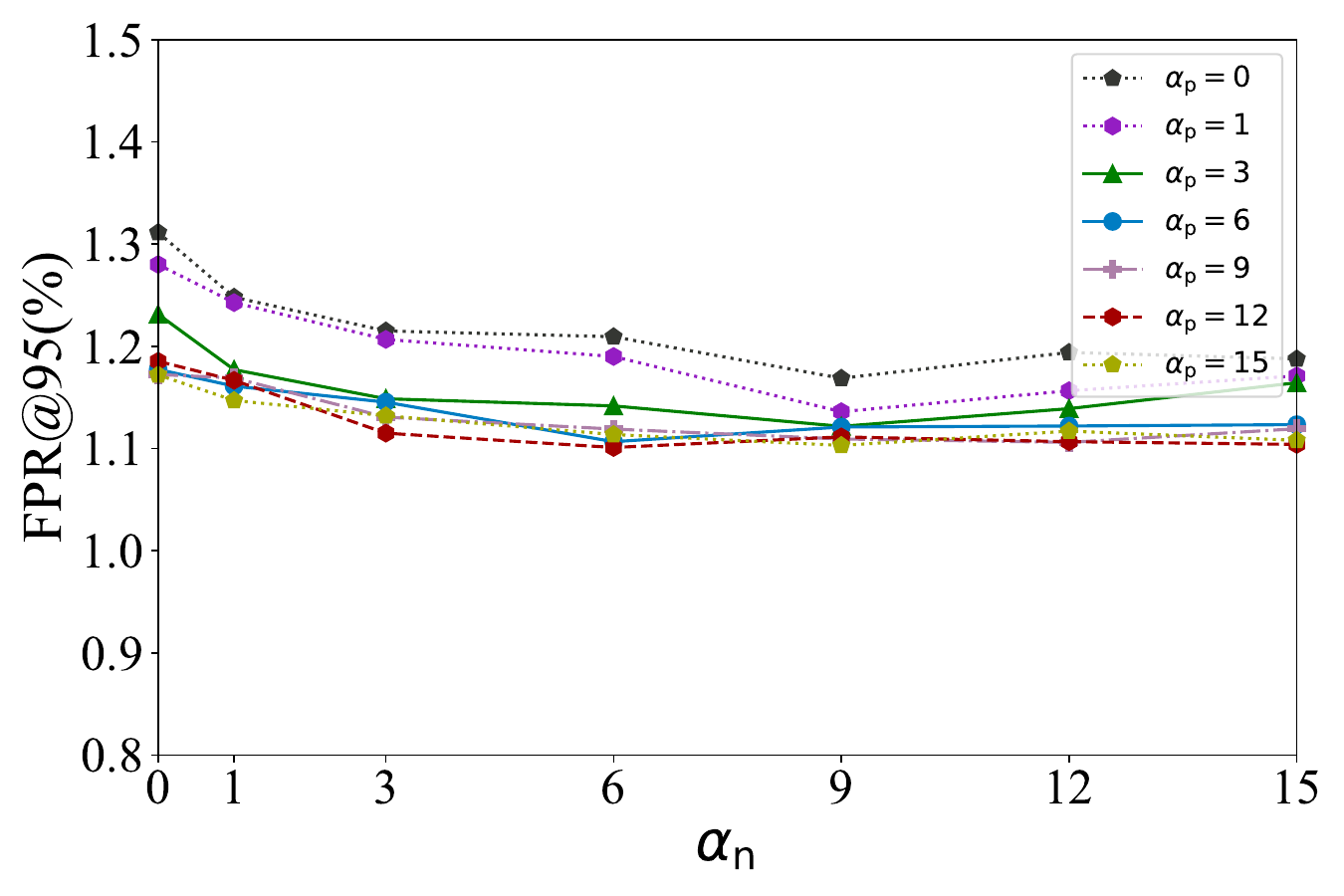}
    \label{fig:pn_light}
  }\\
  \vspace{-0.1in}
  \subfloat[DesDis-HardNet]{
    \includegraphics[width=0.4\textwidth]{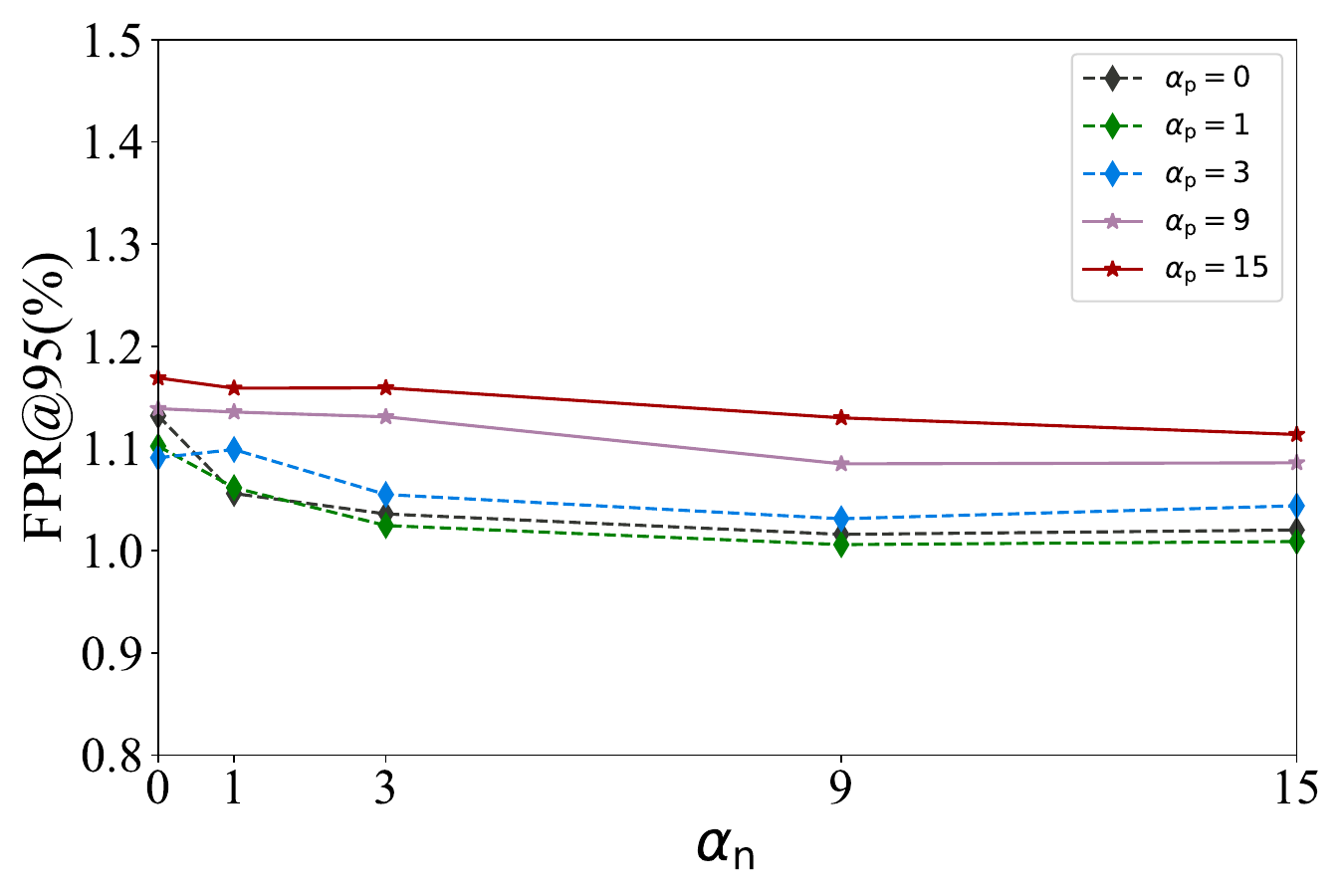}
    \label{fig:pn_equal}
  }
  \caption{
    Impact of $\alpha_\textrm{p}$ and $\alpha_\textrm{n}$ on
    the light-weight student (a) DesDis-32, and
    the equal-weight student (b) DesDis-HardNet.
    The models are trained on the Liberty subset of the Brown dataset,
    and then evaluated on the Yosemite and Notredame subsets.
    The results under each parameter configuration are averaged
    over 5 reruns.
  }
\end{figure}

\subsubsection{Impact of $\alpha_\textrm{p}$ and $\alpha_\textrm{n}$}
In this subsection, we evaluate the effect of the two weights
$\alpha_\textrm{p}$ and $\alpha_\textrm{n}$ in \cref{eq:Lobj}.
It is noted from many visual tasks: when a set of training data is
used to respectively train two networks that have a completely same loss
function (often containing multiple weighted loss terms)
but different architectures (e.g., different numbers of layers),
their performances would be generally different. An appropriate weight
configuration of loss terms for one network might not be quite appropriate for the other.
Hence, we conduct experiments on light-weight and equal-weight
student models respectively.
In the light-weight student case,
we evaluate DesDis-32 under different configurations of
$\alpha_\mathrm{p}=\{0,1,3,6,9,12,15\}$ and $\alpha_\mathrm{n}=\{0,1,3,6,9,12,15\}$.
We train the model on the \textit{Liberty} subset of Brown, and then test it
on the \textit{Yosemite} and \textit{Notredame} subsets.
We evaluate the model under each parameter configuration
five times independently. The corresponding mean results
are shown in \cref{fig:pn_light}.
As seen from this table,
when the two weights vary from 6 to 15, the performance of DesDis-32 varies slowly,
indicating that the proposed TS regularizer is not quite sensitive to the two weights.
So the two parameters are simply set to 9 for the proposed light-weight models in all the experiments.
In the equal-weight student case, we evaluate
DesDis-HardNet under different configurations of
$\alpha_\mathrm{p}=\{0,1,3,9,15\}$ and $\alpha_\mathrm{n}=\{0,1,3,9,15\}$.
We train the model on the \textit{}Liberty subset of Brown, and then test it on
the \textit{Yosemite} and
\textit{Notredame} subsets. We evaluate the model under each parameter configuration
five times independently. The corresponding mean results are shown in \cref{fig:pn_equal}.
As seen from this figure,
when $\alpha_\mathrm{p}=1$ and $\alpha_\mathrm{n}$ varies from 9 to 15,
the model varies slightly and performs better in comparison to the
other parameter configurations. So, $\alpha_\mathrm{p}$ is set to 1 and  $\alpha_\mathrm{n}$
is set to 15 for the equal-weight models in all the experiments.

\begin{table}[t]
  \centering
  \fs
  \caption{Impact of $\mathcal{L}_\mathrm{B}$.
    The models are trained on the
    Liberty subset of Brown, and tested on the Yosemite (YOS)
    and Notredame (ND) subsets.
    The numbers are false positive rate at 95\% recall.
  }
  \label{tab:lB}
  \begin{tabular}{lccc}
    \toprule
    Method                                       & YOS$\downarrow$ & ND$\downarrow$ & Mean$\downarrow$ \\
    \midrule
    HyNet                                        & 0.88            & 0.34           & 0.61             \\
    DesDis-HyNet (w/o $\mathcal{L}_\mathrm{B}$)  & 0.93            & 0.35           & 0.64             \\
    DesDis-HyNet (with $\mathcal{L}_\mathrm{B}$) & \textbf{0.70}   & \textbf{0.29}  & \textbf{0.50}    \\
    \bottomrule
  \end{tabular}
\end{table}

\begin{table*}[t]
  \centering
  \caption{Comparison among the light-weight student models with $\{3,4,5,6\}$ layers
    (denoted as Net-I, Net-II, Net-III, Net-IV, and Net-III is exactly DesDis-32).
    The numbers in the seven columns on the right are the false positive
    rates at 95\% recall on Brown \citep{BrownDataset}.
    The throughputs are tested on a GTX 1650Ti.}
  \label{tab:fpr}
  \newcolumntype{C}[1]{>{\centering\let\newline\\\arraybackslash\hspace{0pt}}m{#1}}
  \def\gridwidth{0.04}
  \fs
  \begin{tabular}{lcccccccccc}
    \toprule
    {Train}     & \multirow{2}{*}{\centering \#Layers} & \multirow{2}{*}{\centering \#Param.} & Throughputs   & ND                      & YOS                    & LIB                     & YOS     & LIB      & ND       & \multirow{2}{0.07\textwidth}{\centering Mean} \\
    \cmidrule(l{6pt}r{6pt}){5-6} \cmidrule(l{6pt}r{6pt}){7-8} \cmidrule(l{6pt}r{6pt}){9-10}
    Test        &                                      &                                      & (K patch/sec) & \multicolumn{2}{c}{LIB} & \multicolumn{2}{c}{ND} & \multicolumn{2}{c}{YOS} &                                                                               \\
    \midrule
    Net-I       & 3                                    & 0.543M                               & 261           & 3.87                    & 4.78                   & 1.09                    & 1.57    & 3.83     & 3.04     & 3.03                                          \\
    Net-II      & 4                                    & 0.355M                               & 195           & 2.07                    & 2.95                   & 0.61                    & 0.94    & 2.16     & 2.10     & 1.81                                          \\
    \bf Net-III & 5                                    & 0.502M                               & 145           & 1.52                    & 2.36                   & 0.54                    & \bf0.81 & \bf 1.68 & 1.48     & 1.39                                          \\
    Net-IV      & 6                                    & 0.539M                               & 102           & \bf 1.35                & \bf 2.11               & \bf 0.46                & 0.82    & 1.76     & \bf 1.33 & \bf 1.31                                      \\
    \bottomrule
  \end{tabular}
\end{table*}


\subsubsection{Impact of $\mathcal{L}_\mathrm{B}$}
In this subsection, we evaluate the impact of the basic loss term
$\mathcal{L}_\mathrm{B}$ in \cref{eq:Lobj}.
Here, we use HyNet \citep{HyNet} as the teacher model, and derive the corresponding
equal-weight student model that is trained only with the TS regularizer
(without $\mathcal{L}_\mathrm{B}$) in \cref{eq:Lobj}.
The \textit{Liberty} subset of the Brown dataset is used for training,
and the \textit{Yosemite} and \textit{Notredame} subsets are used for testing.
The corresponding results are reported in \cref{tab:lB}.
In addition, the corresponding results by the original HyNet (baseline) and our
complete model DesDis-HyNet with $\mathcal{L}_\mathrm{B}$ are also added into
\cref{tab:lB} for comparison.
As seen from this table, the performance of the model trained without
$\mathcal{L}_\mathrm{B}$ is slightly worse than that of the original HyNet,
and also worse than that of the proposed complete model trained with
both $\mathcal{L}_\mathrm{B}$ and the TS regularizer. These results demonstrate
that the joint usage of both $\mathcal{L}_\mathrm{B}$ and the proposed TS
regularizer is much helpful for improving performances.

\subsubsection{Impact of the Number of Layers in the Light-Weight Models}
\label{sec:architecture}
Considering that we have tested four 5-layer student models with different numbers
of channels in \cref{subsec:lightweight}, we evaluate the performances of
the student models with different numbers of layers in this subsection. Similar to
the experimental configuration in \cref{subsec:lightweight}, we adopt HyNet
\citep{HyNet} as the teacher model. Then, we use the Brown dataset to evaluate
four light-weight student models with \{3,4,5,6\} layers respectively. The
four models are denoted as Net-I, Net-II, Net-III, Net-IV, and Net-III is
exactly the DesDis-32 in previous sections.

\begin{figure}[t]
  \begin{center}
    \includegraphics[width=0.48\textwidth]{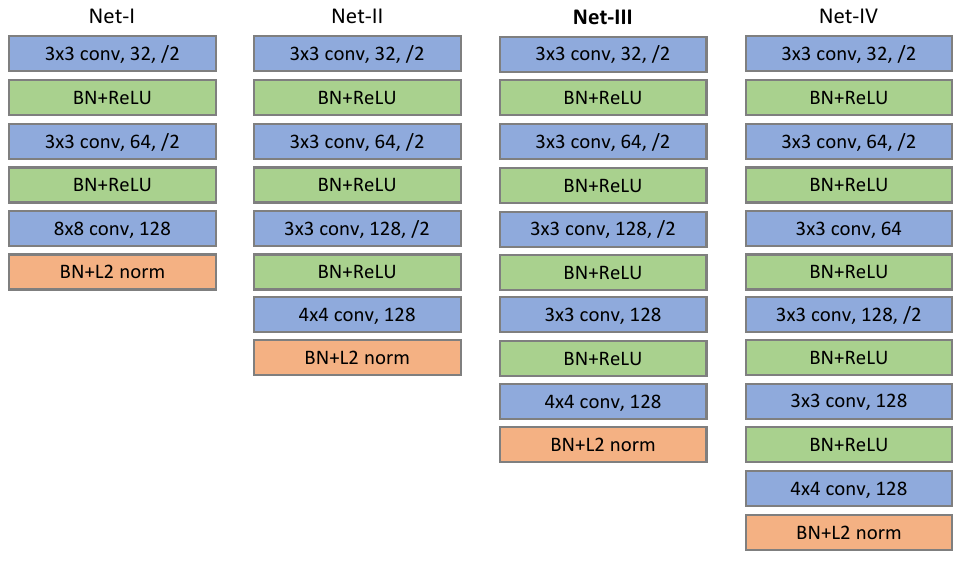}
  \end{center}
  \caption{Architectures of the four models with \{3,4,5,6\} convolutional layers.
    `BN' denotes batch normalization. `/2' denotes convolution operation with a stride of 2.
    Net-III is exactly the DesDis-32 in \cref{sec:experiments}.}
  \label{fig:arc}
\end{figure}

\cref{fig:arc} shows the architectures of the four light-weight models. As
seen from this figure, except the last convolutional layer that encodes the feature
maps into a 128 dimensional descriptor, all the rest convolutional layers in the
four models have a kernel size of 3$\times$3. The channel number in the first
layer is 32, and is doubled whenever the feature map is downsampled through strided
convolution. The third and fourth columns in \cref{tab:fpr}
report the parameter numbers and the throughputs of the four models. It is noted
that among these models, although Net-I has only 3 layers, it has the largest
number of parameters, mainly because it uses a large convolution kernel
8$\times$8 in its last layer. In addition, although Net-I has the largest
number of parameters, its computational cost is the lowest, hence, it has the
largest throughput. In general, with the increase of the layer number, a student
model achieves a relatively smaller throughput.

\begin{table*}[t]
  \caption{
    Evaluation on the loss terms of the DesDis framework by utilizing the
    handcrafted descriptor as teacher.
    $\mathcal{L}_\mathrm{B}$ denotes the basic loss term from HyNet.
    TSR denotes the teacher-student regularizer.
    The numbers denote the false positive rate at 95\% recall.
  }
  \label{tab:SIFT_e2e}
  \def\len{0.25in}
  \centering
  \fs
  \begin{tabular}{lccC{\len}C{\len}C{\len}C{\len}C{\len}C{\len}C{0.3in}}
    \toprule
    Train                                & \multirow{2}{*}{\centering $\mathcal{L}_\mathrm{B}$} & \multirow{2}{*}{TSR} & ND                      & YOS                    & LIB                     & YOS & LIB & ND & \multirow{2}{0.07\textwidth}{\centering Mean} \\
    \cmidrule(l{6pt}r{6pt}){4-5} \cmidrule(l{6pt}r{6pt}){6-7} \cmidrule(l{6pt}r{6pt}){8-9}
    Test                                 &                                                      &                      & \multicolumn{2}{c}{LIB} & \multicolumn{2}{c}{ND} & \multicolumn{2}{c}{YOS} &                                                                \\
    \midrule
    SIFT                                 &                                                      &                      & \multicolumn{2}{c}{29.86} & \multicolumn{2}{c}{22.53} & \multicolumn{2}{c}{27.29} & 26.55 \\
    \#1 DesDis-32 (SIFT as teacher)      &                                                      & $\surd$             & 32.03   & 31.05                  & 19.71 & 20.21 & 25.34 & 25.55 & 25.65                                                                                   \\
    \#2 DesDis-32 (SIFT as teacher)      & $\surd$                                              & $\surd$             & 1.83 & 2.78 & 0.51 & 0.95 & 2.15 & 1.71 & 1.67 \\
    \#3 DesDis-32 (without distillation) & $\surd$                                              &                     & 1.81 & 2.86     & 0.55    & 0.90      & 2.01     & 1.69     & 1.64  \\ 
    DesDis-32 (HyNet as teacher)         & $\surd$                                              & $\surd$             & 1.52                  & 2.36                  & 0.54                  & 0.81  & 1.68 & 1.48 & 1.39                                      \\
    \midrule
    ORB   &          &        & \multicolumn{2}{c}{57.37} & \multicolumn{2}{c}{49.05} & \multicolumn{2}{c}{53.65} & 53.36\\
    Binary DesDis-32 (ORB as teacher) V1  & &$\surd$ & 60.81 & 65.17  & 45.24 & 49.83 & 58.24 & 57.12 & 56.06 \\
    Binary DesDis-32 (ORB as teacher) V2  & &$\surd$ & 58.61 & 58.09  & 48.34 & 45.84 & 53.85 & 52.16 & 52.82 \\
    \bottomrule
  \end{tabular}
\end{table*}

\cref{tab:fpr} also reports the false positive rates at 95\% recall by the
four student models. As noted from this table, with the increase of the layer
number, a student model generally achieves a relatively higher patch verification
performance, but a relatively slower inference speed. Among them, Net-III could
achieve a relatively better balance between computational efficiency and accuracy
($\approx$1.4$\times$ faster than Net-IV without a significant drop in performance).
Thus, we adopt 5-layer-networks in \cref{sec:experiments}.

\subsection{Evaluation of DesDis by Utilizing Handcrafted Descriptor as Teacher}
\label{subsec:handcrafted}

In the previous subsections, we have verified the effectiveness of the proposed
DesDis framework for distilling the knowledge of DNN-based local descriptors.
In this subsection, we further evaluate the effectiveness of DesDis for distilling
the knowledge of the handcrafted descriptor, including 
SIFT \citep{SIFT} and ORB \citep{orb}.

\subsubsection{SIFT Distillation Under DesDis}\label{subsec:sift}
Here, we use the classic SIFT \citep{SIFT} as the teacher and
DesDis-32 as the student. We compare the following variants on the Brown dataset \citep{BrownDataset}:
(1) the model trained with only the proposed TS regularizer, denoted as Model \#1;
(2) the model trained with both the TS regularizer and the basic loss
term $\mathcal{L}_\mathrm{B}$ from HyNet, denoted as Model \#2;
(3) the DesDis-32 that is singly trained with $\mathcal{L}_\mathrm{B}$ (without distillation),
denoted as Model \#3.

The corresponding results are reported at the top of \cref{tab:SIFT_e2e}.
In addition, \cref{tab:SIFT_e2e} also reports the results of SIFT and
the proposed DesDis-32 by utilizing HyNet \citep{HyNet} as teacher for comparison.
The following points could be observed from this table:

(1) Model \#1 that is trained with only the TS regularizer by utilizing SIFT as
teacher could achieve
a slightly better performance than SIFT. This verifies the transferability of
handcrafted descriptor into an end-to-end network to some extent,
however, the improvement brought by the end-to-end network with the TS
regularizer is limited.

(2) Model \#2 that is trained with both the basic loss term $\mathcal{L}_\mathrm{B}$ and the
TS regularizer by utilizing SIFT as the teacher model significantly
outperforms SIFT and Model \#1. However, it is noted
from \cref{tab:SIFT_e2e} that Model \#2 only achieves a close
performance to Model \#3, which is singly trained with $\mathcal{L}_\mathrm{B}$ but does not
use SIFT for distillation. These results indicate that when the handcrafted
descriptor SIFT is used as teacher, the improvement of its student mainly
owes to the basic loss $\mathcal{L}_\mathrm{B}$, rather than the teacher SIFT and
the TS regularizer,
probably because the performance of the teacher SIFT is significantly worse than the
singly trained Model \#3.

(3) The proposed DesDis-32 with both  $\mathcal{L}_\mathrm{B}$ and the TS
regularizer by utilizing HyNet as the
teacher model performs best among all the comparative methods. This result
indicates that when a relatively high-quality model (e.g., HyNet) is
used as teacher,  its student would obtain more benefits from the designed TS regularizer.

\begin{table*}[t]
  \centering
  \caption{
    Comparison of training time (second) on the Liberty subset of Brown.
    The models are trained for 200 epochs on a Tesla V100 GPU.
    The forward propagation when training a teacher only involves
    the forward propagation of itself, while
    the forward propagation when training a student involves
    the forward propagation of itself and its teacher.
  }
  \label{tab:time}
  \fs
  \begin{tabular}{cccccc}
    \toprule
    Model          & Forward Propagation & Loss Computation & Backward Propagation & Total \\
    \midrule
    HardNet        & 1208                & 125              & 3051                 & 4384  \\
    DesDis-HardNet & 2118                & 262              & 3052                 & 5432  \\
    \midrule
    SOSNet         & 1210                & 283              & 3030                 & 4523  \\
    DesDis-SOSNet  & 2118                & 408              & 3042                 & 5568  \\
    \midrule
    HyNet          & 1411                & 153              & 4864                 & 6428  \\
    DesDis-HyNet   & 2453                & 284              & 4865                 & 7602  \\
    \bottomrule
  \end{tabular}
\end{table*}

\subsubsection{ORB Distillation Under DesDis}

In the previous content, we focused on learning and 
distilling the knowledge of  real-valued descriptors.  
In this subsection, we analyze the possibility of 
knowledge distillation of binary descriptors in the proposed 
DesDis framework. Here, we use the classic ORB \citep{orb} descriptor as teacher.

For knowledge distillation of ORB in DesDis, the following issues have to be taken into account: 
(i) the designed Teacher-Student (TS) regularizer in DesDis has to measure the difference 
between the distance of pairwise binary teacher descriptors (i.e., ORB) and  the 
distance of pairwise student descriptors, however, the commonly-used distance metric 
for ORB is the Hamming distance, different from the Euclidean distance metric used 
for the real-valued student descriptors in the current version of DesDis; 
 (ii) it is generally more difficult for a deep network to learn a binary 
 descriptor than a real-valued descriptor in an end-to-end manner from the 
 perspective of network training optimization. 

It is noted that many existing binary descriptor learning methods \citep{L2Net,DOAP} 
employ a two-stage strategy: 
they firstly train a deep network for learning real-valued descriptors, and 
then binarize the real-valued descriptors as a post-processing operation.  
Considering that such a two-stage strategy could avoid the aforementioned second 
issue, we attempt to use this strategy for  knowledge distillation of ORB here. 
Accordingly, in order to handle the aforementioned first issue, we use the 
following two variants of  the TS regularizer respectively for network training: 
(V1) we scale the Hamming distance of pairwise teacher ORB descriptors into 
the range [0, 2] where the Euclidean distances of pairwise normalized student 
descriptors vary, and then measure the difference between the student and teacher 
distances; (V2) we calculate the Euclidean distances (also scaled into the range [0, 2]) 
of pairwise ORB descriptors instead of the Hamming distance, and then measure the 
difference between the student and teacher distances. The corresponding results on 
the Brown dataset \citep{BrownDataset} by utilizing DesDis-32 as the student model are reported in 
\cref{tab:SIFT_e2e} above. As seen from this table, DesDis-32 that is trained with 
the V1 variant performs slightly worse than ORB, 
while DesDis-32 that is trained with the V2 variant performs slightly better than ORB. 
These results demonstrate the possibility of knowledge distillation of binary descriptors to some extent. 

\subsection{Comparison of Training Time of the Teacher and Student}\label{subsec:time}

In this subsection, we evaluate the training time costs of the teacher
and student models. Specifically, we compare the training times of the
three teacher models HardNet \citep{HardNet}, SOSNet \citep{SOSNet} and HyNet \citep{HyNet}, with
their corresponding student models DesDis-HardNet, DesDis-SOSNet and DesDis-HyNet respectively.
All the models are trained on the \textit{Liberty} subset of the Brown dataset
for 200 epochs on a Tesla V100 GPU.
The corresponding training time costs of these models are reported in \cref{tab:time}.

As seen from the table, the time cost for training a model consists
of three parts:
the forward propagation time, the loss computation time, and the back propagation time.
And three points could be observed:

(1) Each proposed distillation method takes nearly twice forward propagation times as long as
its corresponding baseline method,
because the forward propagations of both the student and teacher models have to
be implemented when training the proposed distillation method.

(2) The loss computation time of each proposed method is larger than that of its
corresponding baseline method, because the additional TS regularizer
has to be computed in the proposed method.
However, the training times of both the proposed method and its baseline
for loss computation are much smaller than those for forward and backward propagations.

(3) The training times of each proposed method and its baseline for back propagation
are close, because the student model in the proposed equal-weight method
(where the teacher model is always frozen) has the same architecture as its baseline.
Additionally, the back propagation times of each
proposed method and its baseline are much larger than their forward
propagation times, mainly because of the gradient computations involved during
the back propagation process.

In sum, training an equal-weight student model takes around 1.2 times
as long as training a teacher model.


\section{Conclusion}\label{sec:conclusion}
In this paper, we focus on learning a fast and discriminative local descriptor.
Many existing works
in literature employed a triplet loss or its variants that are expected to
enforce a small distance between positive pairs and a large distance between
negative pairs. However, such an expectation has to be lowered due to the
non-perfect convergence of the networks. Addressing this issue, we propose a
descriptor distillation framework, named DesDis, for local descriptor learning,
where a student model gains knowledge from its teacher. In addition, we
prove in theory that the student model could outperform its teacher through
the designed teacher-student regularizer.
Under the proposed DesDis framework, both equal-weight and light-weight student
models could be obtained. Extensive experimental results on 3 public datasets
demonstrate that the proposed DesDis framework could not only boost the
performances of the existing descriptor learning networks by utilizing them
as the teacher models, but also pursue much lighter student models that
achieve a balance between computational effectiveness and efficiency.

\noindent\textbf{Acknowledgment }
This work is funded by the National Natural Science 
Foundation of China (Grant Nos. 61991423, 62376269) and the Strategic Priority 
Research Program of the Chinese Academy of Sciences (Grant No. XDA27040811).

\noindent\textbf{Data Availability Statement }
The public datasets used in this paper are:
(a) the Brown dataset \citep{BrownDataset},
(b) the HPatches dataset \citep{HPatches}, and
(c) the ETH SfM dataset \citep{ETH}.
(a) is available at \url{http://matthewalunbrown.com/patchdata/patchdata.html},
(b) is available at \url{https://github.com/hpatches/hpatches-dataset}, and
(c) is available at \url{http://www.cvg.ethz.ch/research/local-feature-evaluation/}

\backmatter

\bibliographystyle{apalike}
\bibliography{ref}


\end{document}